\pdfoutput=1
\documentclass[11pt]{article}
\usepackage{lmodern}
\usepackage[T1]{fontenc}
\usepackage[left=31mm,right=31mm,top=33mm,bottom=20mm]{geometry}

\usepackage[lined,ruled]{algorithm2e}

\usepackage{natbib}
\usepackage{booktabs}
\usepackage{graphicx}

\usepackage{enumitem}

\usepackage[usenames,dvipsnames]{xcolor}

\bibpunct{(}{)}{;}{a}{,}{,}

\usepackage{amsmath,amssymb,amsfonts}
\usepackage{amsthm}
\usepackage{thmtools}

\usepackage{bm}

\newtheorem{theorem}{Theorem}
\newtheorem{lemma}{Lemma}
\newtheorem{corollary}{Corollary}
\newtheorem{proposition}{Proposition}

\theoremstyle{definition}

\usepackage[colorlinks,
            linkcolor=blue,
            citecolor=blue,
            urlcolor=magenta,
            linktocpage,
            plainpages=false]{hyperref}

\definecolor{darkgreen}{rgb}{0,0.5,0}

\usepackage[capitalize,nameinlink]{cleveref}

\crefname{definition}{Definition}{Definitions}

\DeclareMathOperator*{\nth}{^{\text{th}}}

\DeclareMathOperator*{\argmin}{arg\,min}

\DeclareMathOperator{\E}{\mathbb{E}}
\DeclareMathOperator{\Var}{\mathsf{Var}}
\DeclareMathOperator{\V}{\mathbb{V}}

\renewcommand{\Pr}{\operatorname{\rm Pr}\nolimits}
\newcommand{\ind}[1]{\delta_{\{#1\}}}

\newcommand{\reals}{\mathbb{R}}

\newcommand{\Y}{\mathcal{Y}}

\newcommand{\loss}{\ell}

\newcommand{\pcalibeat}{p}

\title{Calibeating for General Proper Losses: \\A Bregman Divergence Approach}

\author{Maximilian Fichtl\thanks{Independent researcher, \texttt{mfichtl@gmail.com}}, Crist\'obal Guzm\'an\thanks{Institute for Mathematical and Computational Engineering, Faculty of Mathematics and School of Engineering, Pontificia Universidad Cat\'olica de Chile, \texttt{crguzmanp@uc.cl}}, Nishant A.~Mehta\thanks{Department of Computer Science, University of Victoria, \texttt{nmehta@uvic.ca}}}

\begin{document}

\maketitle

\colorlet{stabcolor}{Cerulean}

\definecolor{pencolor}{rgb}{0.9,0,0}

\definecolor{smoothcolor}{rgb}{0,0.65,0}

\begin{abstract}
This work introduces a general framework for calibeating based on regret minimization. As compared to Foster and Hart's seminal calibeating work which had specialized treatments of Brier score (squared loss) and log loss, we consider a large family of proper losses that includes $\alpha$-Tsallis losses (for $\alpha \in [1, 2]$) and Lipschitz losses. Our results for Tsallis losses also hold for an unscaled version of Tsallis loss that recovers log loss. Our analysis is oriented around the Bregman divergence view of a proper loss. Technically, our results for the family of Tsallis losses that we consider are U-calibration results, simultaneously obtaining logarithmic regret for all losses in this family while having a weaker dependence on the dimension compared to previous results. Of potential independent interest, we also show a new regret equality for the regret of Be The Regularized Leader. This regret equality holds for general proper losses and itself is based on two results related to online updating formulas for the generalized variance, the latter being a previously introduced generalization of variance based on Bregman divergences.
\end{abstract}

\section{Introduction} \label{sec:intro}

Consider forecasting a sequence of outcomes $(y_t)_{t \geq 1}$, taken from a finite set $\Y$, when the outcomes are selected by an adaptive adversary. When the predictions are probability forecasts $(q_t)_{t \geq 1}$ and the loss function is the squared loss $\loss(p, y) = \sum_{j=1}^d \left( p_j - \mathbf{e}_y \right)^2$, it is well-known that the cumulative squared loss can be decomposed into the sum of the \emph{calibration score} --- a measure of the extent to which the forecaster is not calibrated --- and the \emph{refinement score} --- representing the fundamental variance present within the outcome sequence. To \emph{calibeat} is to reduce (or beat) the cumulative loss of a forecaster by an amount equal to the forecaster's calibration score. In their seminal work, \cite{foster2023calibeating} showed that online calibeating strategies do exist. Regarding the squared loss (also called the Brier score in their work), they mentioned:
\begin{quote}
``it easily leads to useful decompositions, such as the Brier score being the sum of the refinement and calibration scores here. However, it raises the question of how much our results depend on the quadratic scores. We believe that the ideas and approach here carry through for other scoring functions\dots''
\end{quote}
Foster and Hart go on to show a similar decomposition is possible for the log loss, leaving open how to push through calibeating for general proper losses.

Our work shows how to achieve calibeating for general proper losses by making a fundamental connection to no-regret online learning.\footnote{In a concurrent work, \cite{chen2026calibeating} independently made this same observation; for more discussion, see the end of \cref{sec:related-work}.} A key player in making this connection is the Bregman divergence viewpoint of a proper loss. While this viewpoint is well-understood \citep{grunwald2004game,buja2005loss,reid2011information}, 
we are unsure if the general community is aware of how neatly calibeating can be connected to regret minimization through the language of Bregman divergences. 
Moreover, in the course of making this connection, we discovered that we had devised a \emph{U-calibration} procedure. 

U-calibration, recently introduced by \cite{kleinberg2023ucalibration}, is the problem of finding a single algorithm (that generates a single sequence of predictions) that simultaneously obtains low regret for all losses in some class of losses. The original U-calibration paper \citep{kleinberg2023ucalibration} considered the class of all bounded, proper losses, as did a follow-up paper by \cite{luo2024optimal}. Our regret minimization results are also U-calibration results. Unlike previous results, the class of losses we consider includes the set of \emph{unscaled $\alpha$-Tsallis losses}, 
defined for any $\alpha \in [1, 2]$ as
\begin{align*}
\loss_\alpha(p, y) = \frac{1}{1 - \alpha} \left( \alpha p_y^{\alpha - 1} - 1 \right) + \sum_{j=1}^d p_j^\alpha .
\end{align*}
We use the term ``unscaled'' because prior works \citep{dawid2007geometry,ruli2022robust,luo2024optimal} instead consider the \emph{scaled $\alpha$-Tsallis loss}, defined for any $\alpha \in [1, 2]$ as\footnote{Throughout the paper, we use the same notation for the unscaled and scaled versions; the text always clarifies which version is meant.}
\begin{align*}
\loss_{\mathrm{\alpha}}(p, y) = (\alpha -  1) \sum_{j=1}^d p_j^\alpha - \alpha p_y^{\alpha - 1} .
\end{align*}
Our interest in the unscaled family is that it includes as special cases the log loss (corresponding to the limiting case of $\alpha = 1$, henceforth referred to as $\alpha = 1$ for simplicity) and the squared loss (corresponding to $\alpha = 2$). Our results also imply regret bounds for the scaled $\alpha$-Tsallis losses for $\alpha \in [1, 2]$ (see \cref{sec:tsallis-family} for our full discussion of unscaled and scaled Tsallis losses). 
We note, however, that the family we consider does not include the 0-1 loss. On the other hand, our results seem to be the first that attain U-calibration on classes of non-Lipschitz, unbounded losses. 

The calibeating procedure we present is based on an algorithm that achieves U-calibration. Consequently, this calibeating procedure simultaneously calibeats for all losses in some class, a notion that we call \emph{U-calibeating}. Our U-calibration algorithm itself fits into the Follow The Regularized Leader (FTRL) framework \citep{abernethy2009competing}, as does our analysis. In particular, by using a generalized variance identity for Bregman divergences \citep{gupta2022ensembles}, we get a regret equality (\cref{lemma:btrl-regret-equality}) for the regret of Be The Regularized Leader relative to the best prediction in hindsight. This regret equality allows us to extend the ideas of \cite{foster2023calibeating} for analyzing squared loss (and log loss) to the wider class of losses that we consider here. The regret equality holds for any proper loss and may be of independent interest.

Our core technical contributions are:
\begin{itemize}
\item A general, regret minimization-based analysis framework for calibeating.
\item An apparently new regret equality (\cref{lemma:btrl-regret-equality}, already mentioned above) for the regret of BTRL versus the best prediction in hindsight, holding for general proper losses. Behind this regret equality are two results that may be of independent interest: \emph{(a)} a generalization of the variance update formula using a generalized notion of variance for Bregman divergences as well as general weights (see \cref{lemma:generalized-sn-thing} in \cref{sec:proof-sketch-btrl-regret-equality}), and, relatedly, \emph{(b)} a generalization of the online formula for the variance (see \cref{lemma:weighted-sum-divergences}, also in \cref{sec:proof-sketch-btrl-regret-equality}).
\item A single, simple algorithm that simultaneously obtains $O(\log T)$ regret for the family of unscaled $\alpha$-Tsallis losses (which includes log loss and squared loss), the scaled $\alpha$-Tsallis losses (which are of interest in the robust statistics community), and all Lipschitz losses.
\item For unscaled and scaled $\alpha$-Tsallis losses, using our general analysis, the $T$-dependent term has the optimal dimension dependence for log loss, which is $d \log T$, and for squared loss, which is simply $\log T$. Moreover, for any $\alpha \in [1, 2]$, the $T$-dependent term is of order $d^{2-\alpha} \log T$, which we suspect is the optimal rate for the $\log T$ term. Our result improves on results of \cite{luo2024optimal}, who give the rate $O(d^2 \log T)$ for any scaled $\alpha$-Tsallis loss and who do not consider the unscaled $\alpha$-Tsallis losses (thereby not being able to handle log loss).
\item Our general analysis for unscaled and scaled $\alpha$-Tsallis losses gives an additive $d$ term, which should not be present for squared loss. For Lipschitz losses with Lipschitz constant $G$, when using no regularization the algorithm obtains regret $O(G \log T)$. For squared loss, this regret bound is $O(\log T)$ with no additive $d$ factor.
\end{itemize}

The next section formalizes the problem setting. We present some background and related work in \cref{sec:related-work}. Next, in \cref{sec:analysis-framework}, we introduce the structure of our analysis framework, showing the connection between calibeating and regret minimization and giving an overview of our U-calibration regret analysis. In \cref{sec:regret-bounds}, we establish U-calibration results by exhibiting a simple FTRL procedure and showing that it simultaneously obtains low regret for various subclasses of losses. We then move to our U-calibeating results in \cref{sec:calibeating}. In \cref{sec:proof-sketch-btrl-regret-equality}, we sketch a proof of our regret equality result (\cref{lemma:btrl-regret-equality}). Finally, \cref{sec:discussion} closes with a discussion and open questions. All proofs not presented in the main text appear in the appendix.

\section{Problem Setup} \label{sec:problem-setup}

In this section, we first formally introduce proper losses and calibeating and then review how to represent proper losses in terms of convex functions and, similarly, in terms of a Bregman divergences.

\paragraph{Proper losses.}
We consider losses of the form $\loss: \Delta_d \times [d] \rightarrow \reals\cup\{+\infty\}$.\footnote{The main reason we include $+\infty$ in the range is because of log loss.} 
We extend the second argument of $\loss$ to the simplex $\Delta_d$ via 
$\loss(p, q) = \E_{Y \sim q} [ \loss(p, Y) ]$. When the second argument of the loss corresponds to a label $y$, we may write either $y$ or $\mathbf{e}_y$, where $\mathbf{e}_y$ is the $y\nth$ standard basis vector. A loss is \emph{proper} if for all $q \in \Delta_d$, the prediction $q$ is a minimizer of $p \mapsto \loss(p, q)$. Moreover, a loss is \emph{strictly proper} if $q$ is a unique minimizer. Throughout this work, we consider proper losses, and we implicitly assume that perfect prediction results in zero loss, i.e.,
\begin{align}
\loss(\mathbf{e}_y, y) = 0 \quad \text{for all } y \in \Y . \label{eqn:fair}
\end{align}

\paragraph{Calibeating.}
Before introducing calibeating, we require some notation.  
For any $t \in [T]$, we denote the empirical outcome frequency up until round $t$ as $\bar{p}_t = \frac{1}{s} \sum_{s=1}^t \mathbf{e}_{y_s}$. 
For any $x \in \Delta_d$ and sequence of forecasts $(q_t)_t$, let $\mathcal{T}_x = \{t \in [T] : q_t = x\}$ and $\bar{p}_T(x)$ be the empirical outcome frequency among the rounds in $\mathcal{T}_x$. Similarly, for any set $B \subseteq \Delta_d$, let $\mathcal{T}_B = \{t \in [T] : q_t \in B\}$.

Following \cite{foster2023calibeating}, we introduce calibeating in the case of the squared loss $\loss(p, y) = \left\| p - \mathbf{e}_y \right\|_2^2$. For a forecaster with sequence of probability forecasts $(q_t)_{t \in [T]}$ against outcome sequence $(y_t)_{t \in [T]}$, the cumulative squared loss $\mathcal{L}_T = \sum_{t=1}^T \loss(q_t, y_t)$ admits the decomposition
$\mathcal{L}_T = \mathcal{L}_T^{(\mathrm{ref})} + \mathcal{L}_T^{(\mathrm{cal})}$, where
\begin{align*}
\mathcal{L}_T^{(\mathrm{ref})} = \sum_{t=1}^T \left\| \mathbf{e}_{y_t} - \bar{p}_T(q_t) \right\|_2^2
\end{align*}
is the refinement score and
\begin{align*}
\mathcal{L}_T^{(\mathrm{cal})} = \sum_{t=1}^T \left\| \bar{p}_T(q_t) - q_t \right\|_2^2
\end{align*}
is the calibration score.

Now, we say that a sequence $(\pcalibeat_t)_{t \in [T]}$ calibeats $(q_t)_{t \in [T]}$ if
\begin{align*}
\sum_{t=1}^T \loss(q_t, y_t)
- \sum_{t=1}^T \loss(\pcalibeat_t, y_t) 
\geq \mathcal{L}_T^{(\mathrm{cal})} + o(T) ,
\end{align*}
where $\mathcal{L}_T^{(\mathrm{cal})}$ is the calibration score under $(q_t)_{t \in [T]}$.

In \cref{sec:calibeating-warmup}, we show that the concept of calibeating --- developed above for squared loss --- can be extended to arbitrary proper losses.

\paragraph{Representation of proper loss in terms of convex function.}

For a convex function $\psi:\Delta_d\mapsto \overline{\mathbb{R}}$ we define its {\em Bregman divergence}  $D_{\psi}:\Delta_d\times\Delta_d\mapsto\overline{\mathbb{R}}$ as
\[ D_{\psi}(p,q)=\psi(p)-\psi(q)-\langle \nabla \psi(q),p-q\rangle \qquad \forall p\in \Delta_d,\,\forall q\in \mbox{dom}(\partial \psi) , \]
where $\nabla \psi$ is a (measurable) selection of subgradients of $\psi$; in particular, $\nabla \psi(p)\in \partial \psi(p)$, for all $p\in \mbox{dom}(\partial \psi)$. One of the most useful properties of Bregman divergences is the three-points identity \citep{chen1993convergence},
\begin{align} \label{eqn:three_points_identity}
D_\psi(p, q) = D_\psi(p, r) + D_\psi(r, q) + \bigl\langle \nabla \psi(r) - \nabla \psi(q), p - r \bigr\rangle .
\end{align}

Next, we review the Savage representation of a proper loss \citep{savage1971elicitation,gneiting2007strictly,reid2011information}. 
\begin{theorem}[\cite{gneiting2007strictly}] \label{thm:Savage}
A loss function $\loss$ is (strictly) proper if and only if there exists a (strictly) convex function $\psi$ and a measurable subgradient selection $\nabla \psi$ such that for all $p\in \Delta_d$ and $\hat p\in \mbox{dom}(\partial \psi)$,
\begin{align}
\loss(\hat{p}, p) = -\psi(\hat{p}) - \langle \nabla \psi(\hat{p}), p - \hat{p} \rangle . \label{eqn:savage}
\end{align}
\end{theorem}
We call the function $\psi$ the {\em convex representation} of the proper loss $\ell$. This representation enjoys some properties that will be needed throughout.
\begin{proposition} \label{prop:basic_facts_Savage}
The following properties hold for any proper loss $\ell$ with convex representation $\psi$.
\begin{enumerate}[label=(\roman*)]
\item For all $y\in {\cal Y}$, $\psi(\mathbf{e}_y)=0$.
\item For all $p\in\Delta_d$ and $y\in {\cal Y}$, $\ell(q,y)=D_{\psi}(\mathbf{e}_y,p)$.
\end{enumerate}
\end{proposition}
\begin{proof}
To see \emph{(i)}, by the Savage representation, 
$-\psi(p)=\ell(p,p)=\min_{\hat p\in \Delta_d} \ell(\hat p,p).$ Now, since the loss is proper and we assume \eqref{eqn:fair}, 
for $p=\mathbf{e}_y$, we have that 
$\ell(\mathbf{e}_y,y) = 0$, and thus $\psi(\mathbf{e}_y)=0$.

We now show \emph{(ii)}. For all $q \in \Delta_d$ and $y \in \Y$, using \eqref{eqn:savage} gives
\begin{align*}
\loss(q, y) 
&= -\psi(q) - \langle \nabla \psi(q), \mathbf{e}_y - q \rangle \\
&= \psi(\mathbf{e}_y) - \psi(q) - \langle \nabla \psi(q), \mathbf{e}_y - q \rangle \\
&= D_\psi(\mathbf{e}_y, q) ,
\end{align*}
where the second line just uses the fact that $\psi(\mathbf{e}_y) = 0$ from \emph{(i)}.
\end{proof}

\section{Background and Related work}
\label{sec:related-work}

\paragraph{Proper losses.}
Proper losses were first introduced in the special case of Brier loss by \cite{brier1950verification} and then more generally by \cite{mccarthy1956measures}. These short papers were succeeded by longer expositions by \cite{shuford1966admissible}, \cite{savage1971elicitation}, \cite{buja2005loss}, \cite{gneiting2007strictly}, and \cite{reid2011information}. 
In the binary case, \cite{shuford1966admissible} characterized proper losses via what is now commonly called the Savage representation of a proper loss. \cite{savage1971elicitation} obtained his eponymous representation in the multi-class case. In a fundamental paper, \cite{gneiting2007strictly} provided a more rigorous treatment of the Savage representation. The explicit connection between proper losses and Bregman divergences was first made --- essentially in parallel --- by \cite{gneiting2007strictly}, \cite{grunwald2004game}, and \cite{buja2005loss}, with the latter authors focusing on the binary case.

\cite{degroot1983comparison} showed for binary outcomes that the expectation of a proper loss can be decomposed into the refinement score and the calibration score. \cite{brocker2009reliability} showed\footnote{See equation (10) of \cite{brocker2009reliability}; note that $d(\cdot, \cdot)$ is a Bregman divergence but with the order of the arguments reversed, and $S(\cdot, \cdot)$ is a loss.} this decomposition holds for general proper losses. Here we review this decomposition to help with the understanding of \cref{sec:calibeating-warmup}. Let the forecast $R$ be a random variable which can (and, if it is any good, should) depend on the outcome $Y$, also viewed as a random variable. Assume that $Y$ has law $P$ and let $P_r$ denote the conditional law of $Y$ given $R = r$. Then the expected loss conditional on report realization $r$ is
\begin{align*}
\loss(r, P_r) 
&= \E [ \loss(r, Y) \mid R = r] \\
&= \E_{Y \sim P_r} [ \loss(P_r, Y) \mid R = r] + D_\psi(P_r, r) \\
&= \loss(P_r, P_r) + D_\psi(P_r, r) ,
\end{align*}
where $\loss(P_r, P_r)$ is the conditional refinement score and $D_\psi(P_r, r)$ is the conditional calibration score, both conditional on $R = r$. 
We can then express the (unconditional) expectation of the loss of $R$ in terms of the (unconditional) refinement and calibration scores:
\begin{align*}
\E [ \loss(R, Y) ] 
= \E [ \loss(P_R, P_R) ] + \E [ D_\psi(P_R, R) ] .
\end{align*}
This decomposition is completely analogous to the decomposition in the online setting, shown both by \cite{foster2023calibeating} for Brier loss and log loss and shown by us (\cref{sec:calibeating-warmup}) for general proper losses.

\paragraph{Follow The Regularized Leader.}
FTRL has too rich of a history to cover here; we refer the reader to the excellent history given in bibliographical notes of \cite{orabona2019modern}. One standard analysis of FTRL's regret is to equate the regret with the sum of two regrets: the regret of FTRL versus BTRL (which is a type of stability term) and the regret of BTRL (versus the best prediction in hindsight). When the loss is Lipschitz, a simple analysis of the stability term involves bounding each iterate by the learning rate $\eta$ times the dual norm of a gradient. For proper losses, a tighter analysis of the stability term is possible: properness allows for a closed form expression for FTRL and BTRL, and again due to properness, the predictions of FTRL and BTRL in a given iterate are close. This tighter analysis is implicit in the analysis of \cite{foster2023calibeating} and is used in the FTL analysis of \cite{luo2024optimal}.

Regarding the regret of BTRL, the standard analysis is to apply the ``Be The Leader'' Lemma to the regularized losses, giving an upper bound equal to the difference between the regularizer's maximum value and the regularizer's evaluation at FTRL's initial prediction. The Be The Leader Lemma is originally due to \cite{hannan1957approximation} (see also see also Lemma 3.1 of \cite{cesabianchi2006prediction}). For proper losses, a tighter analysis of BTRL's regret is also possible. In the case of Brier loss, \cite{foster2023calibeating} showed how BTRL's regret (although they did not use this language) can be expressed using the online formula for the variance. This formula can be proved using the law of total variance. \cite{gupta2022ensembles} showed a generalized law of total variance that holds for Bregman divergences; we use this result to obtain a generalized online formula for the variance and ultimately extend Foster and Hart's expression for BTRL's regret in terms of this generalized online formula for the variance. The generalized law of total variance shown by \cite{gupta2022ensembles} is of intrinsic interest, but Gupta et al.~emphasized the importance of this law for separating out the effect of different sources of randomness on the performance of machine learning algorithms.

\paragraph{Calibration and Calibeating.} 
The calibration literature is too broad to be accurately represented in this review. Nevertheless, it is worth pointing out that the study of calibration in the online setting was initiated by the work of \cite{Dawid:1982}, that established asymptotic calibration for coherent forecasters in a Bayesian setting.  This paper also introduces a strong impossibility result: no online forecaster can calibrate itself against nature in a coherent way. This impossibility was eventually circumvented by the remarkable work of \cite{Foster:1998}, with the introduction of randomized strategies. This line of work later inspired the idea of calibeating \citep{foster2023calibeating}.

\cite{KuleshovErmon:2017}  study online binary  recalibration for proper losses. 
The main distinction between this work and the calibeating literature is that no `beating' is established: namely, the recalibration procedure is approximately calibrated, but its regret is only approximately bounded by  that of the uncalibrated procedure used as input.

Aside from papers mentioned earlier, there are not many works focusing on calibeating or recalibration more broadly. In the offline setting, Platt's scaling is a popular approach \citep{Platt:1999}. This method takes a pretrained classifier as input and recalibrates it on new data by warping within a parametric family of sigmoids. An online recalibration strategy based on Platt's scaling was proposed by \cite{gupta2023online}. We remark that their recalibration strategy is based on Brier score and therefore is not directly related to the interests of our work.

\paragraph{Concurrent works.}
After completing our work, we became aware of two concurrent and independent works, by \cite{chen2026calibeating} and \cite{foster2026proper}, 
which also study calibeating. We discuss these works in turn.

Both \cite{chen2026calibeating} and our work make exactly the same key observation that calibeating can be reduced to no-regret learning. The works diverge at this point. When considering regret that is logarithmic in $T$, \cite{chen2026calibeating} consider using exponentially weighted online optimization \citep{hazan2007logarithmic}; this method is loss-specific, meaning that changing the loss will change the sequence of predictions.\footnote{However, for bounded, proper losses, they consider a prior U-calibration algorithm of \citep{luo2024optimal}.} On the other hand, we consider a simple FTRL-style algorithm that U-calibeats (i.e, simultaneously calibeats) for a class of losses that includes Lipschitz proper losses and $\alpha$-Tsallis losses (for $\alpha \in [1, 2]$). Interestingly, \cite{chen2026calibeating} also show that lower bounds for regret minimization imply lower bounds for calibeating, giving minimax equivalence. They also show results for multi-calibeating and obtaining calibeating predictors that cannot themselves be ``calibeaten''; while very interesting, these results do not overlap with our work and so are out of scope for our comparison.

\cite{foster2026proper} study calibration and calibeating for the class of Lipschitz proper losses. Among their main results that may overlap with our work\footnote{Similar to original paper of \cite{foster2023calibeating}, the new paper of \cite{foster2026proper} also gives results for multi-calibeating and obtaining a calibeating predictor that itself cannot be calibeaten; these results are no doubt interesting but out of scope for our comparison.} are that calibration under squared loss implies calibration under all Lipschitz proper losses and that calibeating under squared loss does not imply calibeating for general, Lipschitz proper losses.
Regarding the first result, they define a notion of \emph{proper calibeating}: a single predictor \emph{proper calibeats} if it simultaneously calibeats for all Lipschitz\footnote{Compactness of the simplex implies these losses are also bounded.} proper losses. They show that a simple predictor, FTL, proper calibeats. Comparing to our work, proper calibeating is equivalent to U-calibeating when the class of losses is all Lipschitz proper losses. We show that the FTRL predictor \eqref{eqn:ftrl} proper calibeats, but more than that, it U-calibeats for an expanded class of losses that includes the $\alpha$-Tsallis family of losses for $\alpha \in [1, 2]$; recall that $\alpha$-Tsallis loss is not Lipschitz for $\alpha \in [1, 2)$. To prove that FTL proper calibeats, \cite{foster2026proper} generalize various supporting technical results of their prior analysis \citep{foster2023calibeating} to accommodate general proper losses. We independently also extend the prior analysis of \cite{foster2023calibeating}, but for FTRL, which introduces some additional challenges.

Ultimately, despite a similar motivation to extend the theory of calibeating to general proper losses, the results in our work and that of \cite{foster2026proper} are mostly independent (when considering their common ancestor of \cite{foster2023calibeating}). Also, in particular, their work does not study non-Lipschitz losses and (relatedly) does not incorporate regularization.

\section{Analysis framework} \label{sec:analysis-framework}

\subsection{From Calibeating to U-Calibration} \label{sec:calibeating-warmup}

To isolate the crux of the connection between calibeating and no-regret learning, in this section we consider a forecaster whose probability forecasts $(q_t)_{t \geq 1}$ belong to a finite set of forecasts $\mathcal{Q} \subset \Delta_d$. For this section only, the reader should think of each forecast from $\mathcal{Q}$ occurring frequently in the sequence $(q_t)_{t \geq 1}$. In general, the probability forecast in each round can be and likely is distinct, even as $T \rightarrow \infty$. We handle this more general setting in \cref{sec:calibeating} by extending our analysis in the present section with a more or less standard binning approach.

We begin with the decomposition
\begin{align*}
\sum_{t=1}^T \loss_t(q_t, y_t) 
= \sum_{x \in \mathcal{X}} \sum_{t \in \mathcal{T}_x} \loss_t(q_t, y_t) .
\end{align*}

\begin{proposition} \label{prop:decomposition}
For any $x \in \mathcal{X}$,
\begin{align}
\sum_{t \in \mathcal{T}_x} \loss(q_t, y_t) 
= \sum_{t \in \mathcal{T}_x} D_\psi(y_t, \bar{p}_T(x)) 
  + \sum_{t \in \mathcal{T}_x} D_\psi(\bar{p}_T(x), x) . \label{eqn:decomposition}
\end{align}
\end{proposition}
\begin{proof}
First, by \Cref{prop:basic_facts_Savage}, we have that
for all $q \in \Delta_d$ and $y \in \Y$, $\loss(q, y)=D_\psi(\mathbf{e}_y, q)$. 
Next, from the three-points identity \eqref{eqn:three_points_identity},
\begin{align*}
D_\psi(y_t, q_t) = D_\psi(y_t, \bar{p}_T(q_t)) + D_\psi(\bar{p}_T(q_t), q_t) + \bigl\langle \nabla \psi(\bar{p}_T(q_t)) - \nabla \psi(q_t), y_t - \bar{p}_T(q_t) \bigr\rangle .
\end{align*}

Let us consider all rounds for which $q_t$ takes some fixed value $x \in \Delta_d$. Using the above expression, the forecaster's cumulative loss in these rounds is
\begin{align*}
&\sum_{t \in \mathcal{T}_{x}} \biggl( 
       D_\psi(y_t, \bar{p}_T(x)) 
       + D_\psi(\bar{p}_T(x), x) 
       + \Bigl\langle \nabla \psi(\bar{p}_T(x)) - \nabla \psi(x), 
                              y_t - \bar{p}_T(x) 
           \Bigr\rangle 
   \biggr) \\
&= \sum_{t \in \mathcal{T}_{x}} \Bigl( D_\psi(y_t, \bar{p}_T(x)) + D_\psi(\bar{p}_T(x), x) \Bigr) ,
\end{align*}
where we used the fact that $\sum_{t \in \mathcal{T}_x} \left( y_t - \bar{p}_T(x) \right) = 0$.
\end{proof}

We call the first summation in the right-hand side of \eqref{eqn:decomposition} the 
\emph{conditional refinement score} (given $x$) because, after summing over $x$, it becomes the (unconditional) refinement score. We call the second summation the \emph{conditional calibration score} (given $x$); again, summing over $x$ yields the (unconditional) calibration score. In the case of squared loss, the conditional refinement score is nothing other than the variance of the outcomes 
among the rounds in which the forecaster played forecast $x$. 
If we think of each forecast $x$ as defining its own bin\footnote{\cref{sec:calibeating} considers the more general setting in which each bin can be associated with multiple forecasts.} --- all rounds where the forecast is made belong to this bin --- then for a fixed binning, this variance of the outcomes within a bin is fundamental to the sequence of outcomes and cannot be reduced by the forecaster. For general proper losses, this irreducible quantity is a generalized entropy. It is fruitful to consider the following normalized version of the conditional refinement score:
\begin{align*}
\frac{1}{|\mathcal{T}_x|} \sum_{t \in \mathcal{T}_x} D_\psi(y_t, \bar{p}_T(x)) .
\end{align*}
Using the simple observation that $D_\psi(y_t, \bar{p}_T) = \loss(\bar{p}_T, y_t)$ and the properness of the loss, the above is equal to
\begin{align*}
\frac{1}{|\mathcal{T}_x|} \sum_{t \in \mathcal{T}_x} \loss(\bar{p}_T(x), y_t) 
= \min_{q \in \Delta_d} \frac{1}{|\mathcal{T}_x|} \sum_{t \in \mathcal{T}_x} \loss(q, y_t) .
\end{align*}
Using this alternate form for the conditional refinement score immediately gives the following re-expression of \Cref{prop:decomposition}.

\begin{proposition} \label{prop:calibration-regret}
It holds that
\begin{align*}
\sum_{t \in \mathcal{T}_x} D_\psi(\bar{p}_T(x), x) 
= \sum_{t \in \mathcal{T}_x} \loss_t(q_t, y_t) 
  - \min_{q \in \Delta_d} \sum_{t \in \mathcal{T}_x} \loss(q, y_t) .
\end{align*}
\end{proposition}

Thus, the conditional calibration score (the second summation on the right-hand side of \eqref{eqn:decomposition}) is precisely the regret of the forecaster for the subgame induced by $\mathcal{T}_x$. From here, our goal is to, in an online way, improve the forecaster's predictions by swapping out the sequence $(q_t)_{t \in \mathcal{T}_x}$ with an improved sequence $(p_t)_{t \in \mathcal{T}_x}$ that obtains as low regret as possible. Put another way, considering the subgame induced by $\mathcal{T}_x$, we wish to beat the original forecaster (i.e., reduce its cumulative loss) by an amount which --- modulo the regret --- is equal to the conditional calibration score. After summing over $x$, we wish to --- modulo the sum (over $x$) of the regrets --- beat the original forecaster by an amount equal to its calibration score. Moreover, to accommodate downstream agents that might care about different proper losses, we would like to choose the new sequence $(p_t)_t$ in a way that is agnostic to the particular proper loss being used. That is, we want the new sequence to obtain low regret for all loss functions in some large class. This problem is exactly the problem of U-calibration.

\subsection{U-Calibration via FTRL}
\label{sec:u-calibration}

Let $\mathbf{p}^T$ denote a sequence of forecasts $(p_t)_{t \in [T]}$, and let $\mathbf{y}^T$ denote a sequence of outcomes $(y_t)_{t \in [T]}$. 
The regret $\mathcal{R}^{(\loss)}(\mathbf{p}^T, \mathbf{y}^T)$ of an online learning algorithm that forms predictions $\mathbf{p}^T$ under outcome sequence $\mathbf{y}^T$ is defined as
\begin{align*}
\mathcal{R}^{(\loss)}(\mathbf{p}^T, \mathbf{y}^T) 
= \sum_{t=1}^T \loss(p_t, y_t) 
  - \min_{p \in \Delta_d} \sum_{t=1}^T \loss(p, y_t) .
\end{align*}
Formally, as it was originally introduced by \cite{kleinberg2023ucalibration}, the U-calibration error of an online learning algorithm that forms predictions $(p_t)_t$ is the supremum, taken over all proper losses $\loss$ with bounded range $[-1, 1]$, of the regret $\mathcal{R}^{(\loss)}(\mathbf{p}^T, \mathbf{y}^T)$, i.e.,
\begin{align*}
\sup_{\loss \colon \Delta_d \times \Y \rightarrow [-1, 1]} 
\mathcal{R}^{(\loss)}(\mathbf{p}^T, \mathbf{y}^T) .
\end{align*}
Note that the single sequence of predictions is assessed against any such bounded proper loss. While this class of bounded proper losses includes popular losses like squared loss, 0-1 loss, and cost-sensitive losses (referred to as V-shaped scoring rules in \cite{kleinberg2023ucalibration}), a notable, popular omission from this class is the log loss. Morever, there is an entire spectrum of losses that interpolate between the log loss and the squared loss; this spectrum is traced out by the family of unscaled $\alpha$-Tsallis losses.

In this work, we consider a variant of U-calibration that, using a single sequence of predictions, seeks to obtain, for each proper loss for which the minimax regret is $O(\log T)$, the optimal regret for that loss. While our focus on $O(\log T)$ regret necessarily rules out the 0-1 loss, this focus allows us to pursue refined dependence on dimension quantities. In particular, we seek improved dependence on the dimension $d$ compared to analyses from previous works.

To set the stage for our algorithm, we begin by considering Follow The Leader (FTL) in the case of Lipschitz losses. 
In round $t$, FTL plays the empirical mean $\bar{p}_{t-1}$. For the class of Lipschitz losses, it is not hard to show that FTL obtains regret that is logarithmic in $T$. Yet, there are various proper losses for which FTL produces linear regret. 
For example, under log loss, it is not hard to show that FTL can obtain infinite (hence, at least linear) regret. 
A standard solution is to somehow restrict predictions to a suitable interior of the simplex. One way to accomplish this restriction is to employ regularization. This is precisely the approach we take here: for each loss, we use an instance of FTRL. Depending on the loss, the choice of regularizer varies, but due to a nice property of proper losses (and the particular choice of regularizer we employ), the effect of the regularization on the prediction $p_t$ will be the same for all losses. That is, the predictions are agnostic to the particular proper loss being used. Yet, for a large class of losses, we will be able to obtain regret bounds that are logarithmic in $T$ with either the optimal (or improved, relative to previous results) dependence on the dimension $d$.

The FTRL algorithm that we use is defined according to
\begin{align} 
p_{t+1} = \argmin_{p \in \Delta_d} \left\{ \sum_{s=1}^t \loss(p, y_s) 
      + \frac{1}{\eta} \Phi(p) \right\} , \label{eqn:ftrl}
\end{align}
where $\Phi(p) = \sum_{j=1}^d \loss(p, \mathbf{e}_j)$. Since the regularizer is a sum of losses, we call this regularizer \emph{loss-based regularization}. Given the form of the regularizer, it is easy to see that $p_{t+1}$ takes the form
\begin{equation*}
p_{j,t+1} = \frac{c_{j,t} + \frac{1}{\eta}}{t + \frac{d}{\eta}} ,
\end{equation*}
where $c_{j,t} = \sum_{s=1}^t \ind{\mathbf{e}_{y_s}=j}$ is the number of times outcome $j$ has occurred by the end of round $t$.

The next proposition provides a regret decomposition that is the foundation for our regret analysis. We express the regret of FTRL as the sum of the regret of FTRL relative to Be The Regularized Leader (BTRL), the regret of BTRL relative to $p_{T+1}$ (the best action in hindsight when regularizing the losses), and the regret of $p_{T+1}$ relative to the best action in hindsight. In the result below, recall that $\bar{p}_T = \frac{1}{T} \sum_{t=1}^T \mathbf{e}_{y_t}$ is the best action in hindsight from the properness of the loss.

\begin{proposition}  \label{prop:doubly-generalized-from-gen-var}
Let $\eta > 0$ and take $p_t$ as in \eqref{eqn:ftrl}. Then we have the following regret equality for FTRL:
\begin{align*}
\sum_{t=1}^T \loss(p_t, y_t) - \sum_{t=1}^T \loss(\bar{p}_T, y_t) 
&= \sum_{t=1}^T \loss(p_t, y_t) - \sum_{t=1}^T \loss(p_{t+1}, y_t) && \text{\emph{(stability)}}\\
&\quad+ \sum_{t=1}^T \loss(p_{t+1}, y_t) - \sum_{t=1}^T \loss(p_{T+1}, y_t) && \text{\emph{(regret of BTRL vs $p_{T+1}$)}} \\
&\quad+ \sum_{t=1}^T \loss(p_{T+1}, y_t) - \sum_{t=1}^T \loss(\bar{p}_T, y_t) && \text{\emph{(smoothing)}} .
\end{align*}
\end{proposition}

In \Cref{sec:regret-bounds}, we will show how to bound each of the terms in 
\cref{prop:doubly-generalized-from-gen-var} for various classes of losses. Here, we comment on our approach for controlling the regret of BTRL with respect to its final iterate $p_{T+1}$ (the best action in hindsight when considering the regularized cumulative loss). The typical way to control this regret is via the Be The Leader Lemma (see e.g.~Lemma 3.1 of \cite{cesabianchi2006prediction}), suitably adapted to accommodate regularization. 
In the special case of loss-based regularization, we introduce in \Cref{lemma:btrl-regret-equality} below an apparently new way to express this regret. The regret representation given in the lemma, being an equality, is stronger than what one gets from the Be The Leader Lemma approach. The key to our lemma is a generalized variance identity for Bregman divergences.

In the next lemma, for $j \in [d]$, we somewhat abuse notation and define $p_{j/d} = \frac{1}{j} \sum_{i=1}^j \mathbf{e}_i$. For clarity, note that $p_{1/d} = \mathbf{e}_1$, $p_{2/d} = \frac{1}{2} (\mathbf{e}_1 + \mathbf{e}_2)$, and $p_{d/d} = p_1 = \frac{1}{d} \mathbf{1}$ (which makes our abuse of notation consistent with our standard notation).

\begin{lemma} \label{lemma:btrl-regret-equality}
For any $\eta > 0$,
\begin{align*}
\sum_{t=1}^T \loss(p_{t+1}, y_t) - \sum_{t=1}^T \loss(p_{T+1}, y_t) 
&= \frac{1}{\eta} \sum_{j=2}^d \loss(p_{T+1}, \mathbf{e}_j)  
   - \frac{1}{\eta} \sum_{j=2}^d \loss(p_{j/d}, \mathbf{e}_j) \\
&\quad - \sum_{j=1}^{d-1} \frac{j}{\eta} D_\psi(p_{j/d}, p_{(j+1)/d}) 
       - \sum_{t=1}^T \left( \frac{d}{\eta} + t - 1 \right) D_\psi(p_t, p_{t+1}) .       
\end{align*}
\end{lemma}

Note that on the right-hand side of the equality, all the summation terms but the first one are nonpositive (either due to the nonnegativity of the loss or the nonnegativity of Bregman divergences). In our regret analysis for unscaled $\alpha$-Tsallis loss, we make use of the first nonpositive summation term, but we have not yet tried to make use of the last two summations (we simply drop them in our analysis). For more detail, see \cref{fn:better} on \cpageref{fn:better}.

The heart of the analysis behind \cref{lemma:btrl-regret-equality} is a generalized law of total variance \citep{gupta2022ensembles} for a generalized notion of variance based on Bregman divergences. From the generalized law of total variance, we derive generalized versions of both the variance update formula and the online formula for the variance. These latter two results may be of independent interest. For more information and a proof sketch of \cref{lemma:btrl-regret-equality}, we direct the interested reader to \cref{sec:proof-sketch-btrl-regret-equality}.

\section{Regret bounds} \label{sec:regret-bounds}

\subsection{Warmup: Lipschitz losses}

As a warm-up, we first show how to use the general FTRL-based regret decomposition \Cref{prop:decomposition} for losses that are $G$-Lipschitz. An equivalent result was shown by \cite{luo2024optimal}.

\begin{proposition} \label{prop:lipschitz}
Assume that for all outcomes $y$, the loss $\loss(\cdot, y)$ is strictly proper and $G$-Lipschitz in a norm $\|\cdot\|$. Let $R$ be the $\|\cdot\|$-diameter of $\Delta_d$. Then taking $\eta =  \infty$ so that FTRL becomes FTL, it holds for all $p \in \Delta_d$ that
\begin{align*}
\sum_{t=1}^T \bigl[ \loss(p_t, y_t) - \loss(p, y_t) \bigr] 
\leq G R (1 + \log T ).
\end{align*}
\end{proposition}
\begin{proof}
Since $\eta = \infty$, Proposition~\ref{prop:decomposition} together with Lemma~\ref{lemma:btrl-regret-equality} gives
\begin{align}
\sum_{t=1}^T \bigl[ \loss(p_t, y_t) - \loss(p, y_t) \bigr] 
&= \sum_{t=1}^T \bigl[ \loss(p_t, y_t) - \loss_t(p_{t+1}, y_t) \bigr]
- \sum_{t=1}^T \left(t - 1 \right) D_\psi(p_t, p_{t+1})
\nonumber \\
&\leq \sum_{t=1}^T \bigl[ \loss(p_t, y_t) - \loss(p_{t+1}, y_t) \bigr] \nonumber \\
&\leq \sum_{t=1}^T G \|p_t - p_{t+1}\| \label{eqn:Lipschitz} ,
\end{align}
where the last line uses the fact that the loss is Lipschitz.

Next, since the loss is proper and $\eta = \infty$, 
we have $p_t = \bar{p}_{t-1}$. Therefore,
\begin{align*}
\|p_t - p_{t+1}\| = \|\bar{p}_{t-1} - \bar{p}_t\| = \frac{1}{t} \|y_t - p_t\| \leq \frac{R}{t} .
\end{align*}
Hence, \eqref{eqn:Lipschitz} is at most
\begin{align*}
G R \sum_{t=1}^T \frac{1}{t} 
\leq G R (1 + \log T ) .
\end{align*}
\end{proof}

We also briefly mention that if instead we set $\eta = 1$ (as we will do for Tsallis losses), then we get a somewhat worse regret bound of $G R (d + \log T)$; this result is \cref{prop:lipschitz-ftrl} in \cref{app:lipschitz-ftrl}.

The squared loss (also called Brier loss) is defined as $\loss_{\mathrm{sq}}(p, y) = \sum_{j=1}^d (p_j - (\mathbf{e}_y)_j)^2$. We can take $\|\cdot\|$ to be the $\ell_1$ norm, so the diameter $R = 2$, and then $G = 2$ is an upper bound on $\ell_\infty$ norm of the gradient, so overall we get the regret bound $4 (1 + \log T)$.

\begin{corollary} \label{cor:squared-loss}
For the squared loss, for all $y \in [d]$, 
\begin{align*}
\|\nabla_p \loss(p, y)\|_\infty \leq 2 .
\end{align*}
Therefore, for FTL, it holds for all $p \in \Delta_d$ that
\begin{align*}
\sum_{t=1}^T \bigl[ \loss(p_t, y_t) - \loss(p, y_t) \bigr] 
\leq 4 (1 + \log T ).
\end{align*}
\end{corollary}
\begin{proof}
The proof is immediate from $\nabla_p \, \|p - \mathbf{e}_y\|_2^2 = 2 (p - \mathbf{e}_y)$, observing that the $\ell_\infty$ norm is at most 2, and then applying \cref{prop:lipschitz}.
\end{proof}

The spherical loss is defined as $\loss(p, y) = 1 - \frac{p_y}{\sqrt{\sum_{j=1}^d p_j^2}}$. For spherical loss, we can take $G = 2 \sqrt{d}$, which gives the following regret bound.

\begin{corollary} \label{cor:spherical-loss}
For the spherical loss, for all $y \in [d]$,
\begin{align*}
\|\nabla_p \loss(p, y)\|_\infty \leq 2 \sqrt{d} .
\end{align*}
Therefore, for FTL, it holds for all $p \in \Delta_d$ that
\begin{align*}
\sum_{t=1}^T \bigl[ \loss(p_t, y_t) - \loss(p, y_t) \bigr] 
\leq 4 \sqrt{d} (1 + \log T ).
\end{align*}
\end{corollary}
\begin{proof}
Observe that
\begin{align*}
\|\nabla_p \, \loss(p, y)\|_\infty 
= \left\|-\frac{\mathbf{e}_y}{\|p\|_2} + \frac{2 p_y}{\|p\|_2^3} p \right\| 
\leq \max\left\{ \left\| \frac{\mathbf{e}_y}{\|p\|_2} \right\|_\infty, 
                  \left\| \frac{2 p_y}{\|p\|_2^3} p \right\|_\infty 
      \right\} .
\end{align*}
The first term in the maximum is at most $\sqrt{d}$. 
We bound the second term as 
\begin{align*}
\frac{2}{\|p\|_2} \cdot \max_{j, y \in [d]} \frac{p_y p_j}{\sum_{i=1}^d p_i^2} 
\leq \frac{2}{\|p\|_2} \frac{p_j^2}{p_j^2} 
= 2 \sqrt{d} .
\end{align*}
This gives the bound on $\|\nabla_p \loss(p, y)\|_\infty$. Applying \cref{prop:lipschitz} gives the regret bound.
\end{proof}

\subsection[Scaled and unscaled $\alpha$-Tsallis loss]{Scaled and unscaled $\bm{\alpha}$-Tsallis loss}
\label{sec:tsallis-family}

The \emph{$\alpha$-Tsallis score} is a loss function that has seen use in the robust statistics literature \citep{ruli2022robust}. 
For $\alpha > 1$, which is the regime of interest in robust statistics, $\alpha$-Tsallis score is defined as \cite[equation 35]{dawid2007geometry}:
\begin{align}
\loss_{\mathrm{\alpha}}(p, y) = (\alpha -  1) \sum_{j=1}^d p_j^\alpha - \alpha p_y^{\alpha - 1} . \label{eqn:tsallis-score}
\end{align}
Note that $\alpha$-Tsallis score is strictly proper.\footnote{\cite{dawid2007geometry} also consider the case of $\alpha < 1$, for which the definition has a sign flip, but we do not consider this regime here.}

When it comes to regret minimization, an interesting property of the $\alpha$-Tsallis score is that in the regime $\alpha \in (1, 2)$, the loss is not Lipschitz. We also note that in this regime, the loss range is contained in $[-1, 1]$. Recently,  \cite{luo2024optimal} showed in their Corollary 3 that for $\alpha \in (1, 2)$, FTL obtains a regret bound of\footnote{\cite{luo2024optimal} scale the loss by a constant $\tilde{c}_K$ (they use $K$ for the dimension) to fit the loss range into $[-1, 1]$. It is easy to see that choosing a scaling constant of $1/2$ suffices, regardless of the choice of $\alpha \in [1, 2]$, and so we can ignore their term $\tilde{c}_K$ when using big-O notation.}
\begin{align}
O(\alpha (\alpha - 1) d^2 \log T) \label{eqn:luo-tsallis-regret} .    
\end{align}

In the limiting case of $\alpha \rightarrow 2$, the $\alpha$-Tsallis score becomes $\loss_2(p, y) = \sum_{j=1}^d p_j^2 - 2 p_y$, which is a (shifted, depending on convention) version of squared loss. The family of $\alpha$-Tsallis scores, while standard in the literature, is somewhat unsatisfying because the other limiting case, of $\alpha \rightarrow 1$, renders the loss trivial. If one shifts the $\alpha$-Tsallis score by adding one and then divides the result by $(\alpha - 1)$, we get the following affine transformation of $\alpha$-Tsallis score:
\begin{align}
\loss_\alpha(p, y) = \frac{1}{1 - \alpha} \left( \alpha p_y^{\alpha - 1} - 1 \right) + \sum_{j=1}^d p_j^\alpha , \label{eqn:our-tsallis}
\end{align}
We call this transformed version of the $\alpha$-Tsallis score the \emph{unscaled $\alpha$-Tsallis loss}. 
Like the scaled $\alpha$-Tsallis loss, if we set $\alpha = 2$ then the unscaled $\alpha$-Tsallis loss becomes the version of squared loss we gave earlier in the paper: $\loss_2(p, y) = \sum_{j=1}^d \left( p_j - (\mathbf{e}_y)_j \right)^2 = \loss_{\mathrm{sq}}(p, y)$. Unlike the scaled $\alpha$-Tsallis loss, for unscaled $\alpha$-Tsallis loss the limiting case of $\alpha \rightarrow 1$ recovers log loss, i.e., $\loss_1(p, y) = -\log p_y = \loss_{\log}(p, y)$. Since log loss is unbounded, it follows that in the range $\alpha \in (1, 2)$, unscaled $\alpha$-Tsallis loss is not bounded in general (although it becomes bounded as soon as $\alpha$ is bounded away from 1 by some positive margin).

In this section, we show regret bounds for FTRL with loss-based regularization under unscaled $\alpha$-Tsallis loss. Since unscaled $\alpha$-Tsallis loss can be transformed into scaled $\alpha$-Tsallis loss by multiplying by $(\alpha - 1)$ and then shifting by 1, our regret bounds for unscaled $\alpha$-Tsallis loss immediately imply the same regret bounds for scaled $\alpha$-Tsallis score after scaling by $(\alpha - 1)$. Despite the fact that the loss range of unscaled $\alpha$-Tsallis loss is unbounded as $\alpha \rightarrow 1$, the next theorem shows that FTRL obtains $O(d \log T)$ regret.

\begin{theorem} \label{thm:tsallis-loss}
Let $\alpha \in [1, 2]$. Then under unscaled $\alpha$-Tsallis loss, taking $p_t$ as in \eqref{eqn:ftrl} with $\eta = 1$, the regret is at most
\begin{align*}
O \left( d + d^{2 - \alpha} \log T \right) .
\end{align*}
\end{theorem}
\begin{proof}
The proof uses the regret decomposition \cref{prop:doubly-generalized-from-gen-var} with the three terms of the decomposition bounded using Lemmas \ref{lemma:stability-term-tsallis-loss}, \ref{lemma:penalty-term-tsallis-loss-improved}, and \ref{lemma:smoothing-term-tsallis-loss} respectively.
\end{proof}

Our bound in Theorem~\ref{thm:tsallis-loss} --- after multiplying by $(\alpha - 1)$ due to the scale difference --- improves the bound \eqref{eqn:luo-tsallis-regret} by improving the dimension dependence from $d^2$ to $d^{2 - \alpha}$, which is always strictly better. Moreover, our theorem covers the fundamentally important case of log loss \citep{vovk2015fundamental}, which cannot be covered by the scaled $\alpha$-Tsallis loss results of \cite{luo2024optimal} (nor their other results, since they restrict to bounded losses). We note that for squared loss ($\alpha = 2$), an optimal bound would have no dimension dependence (see \cref{cor:squared-loss}), whereas our bound still has an additive $O(d)$ gap. We believe this term $d$ is an artifact of our analysis. We conjecture that for all $\alpha \in [1, 2]$, the optimal regret for both unscaled and scaled $\alpha$-Tsallis loss is $O(d^{2-\alpha} \log T)$.

Next, we sketch our proof of Theorem~\ref{thm:tsallis-loss}.

\subsection{Proof sketch of Theorem~\ref{thm:tsallis-loss}}
\label{sec:tsallis-regret-lemmas}

Proofs for all lemmas in this section can be found in \cref{app:tsallis-proofs}.

The next lemma bounds the stability term.

\begin{lemma} \label{lemma:stability-term-tsallis-loss}
For all $\eta > 0$ and $\alpha \in [1, 2]$, the stability term satisfies the bound
\begin{align}
&\sum_{t=1}^T \bigl( \loss(p_t, y_t) - \loss(p_{t+1}, y_t) \bigr) \nonumber \\
&\leq 
\alpha \left[
\log T + 
d^{2 - \alpha} \left(
              \log \left( \frac{\eta T}{d} + 1 \right) 
              + \eta^{2 - \alpha} \log^{\alpha-1} \left( \frac{\eta T}{d} + 1 \right) 
          \right) 
\right] .
\label{eqn:tsallis-loss-stability-contrib}
\end{align}
\end{lemma}

Next, we bound the regret of BTRL relative to $p_{T+1}$.

\begin{lemma} \label{lemma:penalty-term-tsallis-loss-improved}
For all $\eta \geq 0$ and $\alpha \in [1, 2]$, the regret of BTRL relative to $p_{T+1}$ is bounded as
\begin{align}
\sum_{t=1}^T \loss(p_{t+1}, y_t) - \sum_{t=1}^T \loss(p_{T+1}, y_t) 
\leq d + d^{2-\alpha} (\max\{ 1, \log(\eta + 1) \} + \log T) . \label{eqn:tsallis-loss-penalty-contrib-improved}
\end{align}
\end{lemma}

\begin{proof}(Sketch) 
Consider outcome sequence $(\tilde{y}_t)_{t \in [d+T]} = (1, 2, \ldots, d, y_1, y_2, \ldots, y_T)$. 
Let $\tilde{p}_{t+1}$ be the action of Be The Leader (BTL) on this outcome sequence. 

From the BTL Lemma (e.g.~Lemma 3.1 of \cite{cesabianchi2006prediction}),
\begin{align*}
\sum_{t=1}^{d+T} \loss(\tilde{p}_{t+1}, \tilde{y}_t) 
- \sum_{t=1}^{d+T} \loss(p_{T+1}, \tilde{y}_t) 
\leq 0 ,
\end{align*}
which we rearrange as
\begin{align*}
\sum_{t=d+1}^{d+T} \loss(\tilde{p}_{t+1}, \tilde{y}_t) 
- \sum_{t=d+1}^{d+T} \loss(p_{T+1}, \tilde{y}_t) 
\leq 
\sum_{t=1}^d \loss(p_{T+1}, \tilde{y}_t) 
- \sum_{t=1}^d \loss(\tilde{p}_{t+1}, \tilde{y}_t) .
\end{align*}
Unpacking notation, this inequality is equivalent to
\begin{align*}
\sum_{t=1}^T \loss(p_{t+1}, y_t) - \sum_{t=1}^T \loss(p_{T+1}, y_t) 
\leq \sum_{j=1}^d \loss(p_{T+1}, \mathbf{e}_j) - \sum_{j=1}^d \loss(p_{j/d}, \mathbf{e}_j) .
\end{align*}

The right-hand side turns out to be equal to
\begin{align*}
\underbrace{\frac{\alpha}{1-\alpha} \sum_{j=1}^d \left( p_{j,T+1}^{\alpha-1} - (p_{j/d})_j^{\alpha-1} \right)}_{\textcolor{red}{\bigstar}} 
+ \underbrace{d \sum_{i=1}^d p_{i,T+1}^\alpha - \sum_{j=1}^d \sum_{i=1}^d (p_{j/d})_i^\alpha}_{\textcolor{blue}{\bigstar}} .
\end{align*}
For the term $\textcolor{blue}{\bigstar}$, we use the trivial upper bound of $d$.

We claim that the term $\textcolor{red}{\bigstar}$ is $O(d^{2-\alpha} \log T)$. We now sketch a proof of this claim.  After dividing by $d^{2-\alpha}$, this term can be shown to be equal to
\begin{align*}
\frac{1}{d} \cdot \frac{1}{\alpha-1} \sum_{j=1}^d \left( \left( \frac{d}{j} \right)^{\alpha-1}  - \left( \frac{d/\eta}{T + d/\eta} \right)^{\alpha-1} \right) .
\end{align*}
It remains to show that the above is at most $\max\{ 1, \log(\eta + 1) \} + \log T$. To this end, we establish that the expression above is convex in $\alpha$ on the range $\alpha \in [1, 2]$. Hence, it suffices to upper bound the expression at its endpoints, $\alpha = 1$ and $\alpha = 2$. For $\alpha = 1$, we show an upper bound of $1 + \log T$, and for $\alpha = 2$, we show an upper bound of $\log(\eta + 1) + \log T$.
\end{proof}

The dependence on $\eta$ in the above result is strange because as $\eta \rightarrow \infty$, BTRL becomes Be The Leader (BTL) and $p_{T+1}$ becomes the best action in hindsight. Hence, in this case, we should expect a nonpositive regret bound. We attribute our result not being good in this case to our dropping --- in the proof --- certain nonpositive terms from \cref{prop:doubly-generalized-from-gen-var} (which generally seem minor but can become very important when $\eta$ is large). Also, as $\eta \rightarrow 0$, our bound is the best possible, even though BTRL is useless in this case. The way to reconcile this latter fact is that we also compare BTRL to the regularized best-in-hindsight action $p_{T+1}$, so when $\eta = 0$, both BTRL and the regularized best-in-hindsight action are the uniform distribution.

We also give a second result for BTRL's regret relative to $p_{T+1}$. When $\alpha$ is bounded away from $1$ with some margin, this second result is generally better since its bound does not grow with $T$.

\begin{lemma} \label{lemma:penalty-term-tsallis-loss-fallback}
For all $\eta > 0$ and $\alpha \in [1, 2]$, the regret of BTRL relative to $p_{T+1}$ is bounded as
\begin{align}
\sum_{t=1}^T \loss(p_{t+1}, y_t) - \sum_{t=1}^T \loss(p_{T+1}, y_t) 
\leq d \cdot \frac{\alpha}{\alpha-1}. \label{eqn:tsallis-loss-penalty-contrib-fallback}
\end{align}
\end{lemma}

Finally, we control the smoothing term. Unlike the previous results, currently we only give a bound for $\eta = 1$ due to the difficulty in bounding this term.

\begin{lemma} \label{lemma:smoothing-term-tsallis-loss}
Let $\eta = 1$. For all $\alpha \in [1, 2]$, the smoothing term satisfies the bound
\begin{align}
\sum_{t=1}^T \bigl( \loss(p_{T+1}, y_t) - \loss(\bar{p}_T, y_t) \bigr) 
\leq 3 \alpha d \label{eqn:tsallis-loss-smoothing-contrib} .
\end{align}
\end{lemma}
The proof of \cref{lemma:smoothing-term-tsallis-loss} is intricate. We did not see a way to illuminate the proof using a sketch.

\section{Calibeating} \label{sec:calibeating}

Recall that in \cref{sec:calibeating-warmup} we assumed a trivial binning where each distinct forecast is associated with its own bin. In general, it is likely that each forecast in the sequence $(q_t)_{t \geq 1}$ was made only in a single round; consequently, each online learning subgame in \cref{sec:calibeating-warmup} consists of a single round, and no-regret learning becomes useless. 
In the present section, we give our full results for calibeating. To this end, we first present \cref{prop:decomposition-binning-new}, which extends \cref{prop:decomposition} to accommodate binnings for which each bin can contain multiple, distinct forecasts.

Let $\mathcal{B}$ be a finite, disjoint union of bins $B$ such that $\bigcup_{B \in \mathcal{B}} = \Delta_d$ and each bin $B$ has a unique representative $\tilde{x}_B$ in a finite grid $\mathcal{G}$. Thus, in particular, each forecast $x \in B$ is associated with representative $\tilde{x}_B$. For any bin $B \in \mathcal{B}$, recall that $\mathcal{T}_B = \{t \in [T] : p_t \in B\}$.

\begin{proposition}[Generalized version (binning)] \label{prop:decomposition-binning-new}
Let $\mathcal{B}$ be a binning as above where each bin $B \in \mathcal{B}$ is associated with a representative $\tilde{x}_B$. For any bin $B \in \mathcal{B}$,
\begin{align*}
\sum_{t \in \mathcal{T}_B} \loss(q_t, y_t) 
= \sum_{t \in \mathcal{T}_B} D_\psi(y_t, \bar{p}_T(B))
   + \sum_{t \in \mathcal{T}_B} D_\psi(\bar{p}_T(B), \tilde{x}_B) 
   + \sum_{t \in \mathcal{T}_{B}} 
         \bigl( \loss(q_t, y_t) - \loss(\tilde{x}_B, y_t) \bigr) .
\end{align*}
\end{proposition}

Similar to \cref{prop:calibration-regret}, we can re-express the equality in \cref{prop:decomposition-binning-new} in terms of regret-minimization.

\begin{proposition} \label{prop:calibration-regret-binning-new}
It holds that
\begin{align*}
\sum_{t \in \mathcal{T}_B} \loss(q_t, y_t) - \min_{q \in \Delta_d} \sum_{t \in \mathcal{T}_B} \loss(q, y_t) 
= \sum_{t \in \mathcal{T}_B} D_\psi(\bar{p}_T(B), \tilde{x}_B) 
  + \sum_{t \in \mathcal{T}_{B}} 
        \bigl( \loss(q_t, y_t) - \loss(\tilde{x}_B, y_t) \bigr) .
\end{align*}
\end{proposition}

The above proposition makes clear that the conditional calibration score plus a (typically negative) approximation error due to binning is equal to the regret. Thus, if we can improve the forecast sequence $(q_t)_{t \in \mathcal{T}_B}$ to a sequence $(\pcalibeat_t)_{t \in \mathcal{T}_B}$ whose time-average regret (over $\mathcal{T}_B$) goes to zero as $|\mathcal{T}_B| \rightarrow \infty$, then up to the additive approximation error term and the regret of $(\pcalibeat_t)_{t \in \mathcal{T}_B}$, we will have reduced the cumulative loss of the sequence $(q_t)_{t \in \mathcal{T}_B}$ by an amount equal to the conditional calibration score. The next result formalizes this idea.

\begin{corollary}
For any bin $B \in \mathcal{B}$,
\begin{align*}
&\sum_{t \in \mathcal{T}_B} \loss(q_t, y_t) 
 - \sum_{t \in \mathcal{T}_B} \loss(\pcalibeat_t, y_t) \\
&\geq  \sum_{t \in \mathcal{T}_B} D_\psi(\bar{p}_T(B), \tilde{x}_B) 
       - \left|
             \sum_{t \in \mathcal{T}_{B}} 
                 \bigl( \loss(q_t, y_t) - \loss(\tilde{x}_B, y_t) \bigr)
         \right|
       - \mathcal{R}_{\mathcal{T}_B}((\pcalibeat_t)_{t \in \mathcal{T}_B}) ,
\end{align*}
where $\mathcal{R}_{\mathcal{T}_B}((\pcalibeat_t)_{t \in \mathcal{T}_B})$ is the regret of the prediction sequence $(\pcalibeat_t)_{t \in \mathcal{T}_B}$ for the subgame induced by rounds $\mathcal{T}_B$.
\end{corollary}

The regret term in the corollary can be controlled using our U-calibration results (e.g., \cref{thm:tsallis-loss}). The next result controls the approximation error.

\begin{lemma} \label{lemma:binning-new}
For any $\varepsilon \leq \frac{1}{2}$, there exists a binning $\mathcal{B}$ which, for any unscaled $\alpha$-Tsallis loss with $\alpha \in [1, 2]$, satisfies $|\mathcal{B}| = O \left( \left( \frac{\log T}{\varepsilon} \right)^d \right)$ and, for any $t \in \mathcal{T}_B$,
\begin{align*}
\left| \loss(q_t, y_t) - \loss(\tilde{x}_B, y_t) \right|
\leq 
    \begin{cases}
    \varepsilon & \text{for log loss}; \\
    2 \alpha (3 - \alpha) \varepsilon & \text{for unscaled $\alpha$-Tsallis loss}. 
    \end{cases}
\end{align*}
\end{lemma}

\begin{proof}
For each coordinate $j \in [d]$, we form a grid 
\begin{align*}
\left\{
  \left[ \frac{1}{T}, \frac{1}{T} (1 + \varepsilon) \right),
  \left[ \frac{1}{T} (1 + \varepsilon), \frac{1}{T} (1 + \varepsilon)^2 \right),
  \ldots,
  \left[ \frac{1}{T} (1 + \varepsilon)^{N-1}, \frac{1}{T} (1 + \varepsilon)^N \right)
\right\} ,
\end{align*}
where $N$ the smallest integer such that $\frac{1}{T} (1 + \varepsilon)^N \geq 1$.\footnote{Due to our use of regularization, it is fine if $1$ is excluded.} Let $\mathcal{B}$ be the Cartesian product of these $d$ grids, so that each element of $\mathcal{B}$ is a hyperrectangle. 
Note that $N \leq \lceil \log(T) / \log(1 + \varepsilon) \rceil$, and so $|\mathcal{B}| \leq \left( \frac{\log(T)}{\log(1 + \varepsilon)} + 1 \right)^d$. 
Next, we update our definition of $\mathcal{B}$ as follows: for any hyperrectangle $B \in \mathcal{B}$ whose intersection with $\Delta_d$ is empty, we discard $B$ from $\mathcal{B}$. 
For each remaining bin $B \in \mathcal{B}$, we set the representative $\tilde{x}_B$ to be an arbitrarily element of $B \cap \Delta_d$. It follows that any $q_t \in \Delta_d$ belongs to some $B \in \mathcal{B}$ for which, for all $j \in [d]$,
\begin{align}
\frac{(\tilde{x}_B)_j}{(q_t)_j} 
\in \left[ \frac{1}{1 + \varepsilon}, 1 + \varepsilon \right] 
\subset \left[ 1 - \varepsilon, 1 + \varepsilon \right] , \label{eqn:mult-epsilon}
\end{align}
where we used the fact that $\frac{1}{1 + \varepsilon} = 1 - \frac{\varepsilon}{1 + \varepsilon} \geq 1 - \varepsilon$.

The remainder of the proof bounds $\left| \loss(q_t, y_t) - \loss(\tilde{x}_B, y_t) \right|$. To ease notation, we use $q$ in place of $q_t$ and $r$ in place of the representative $\tilde{x}_B$. We also write $y$ instead of $y_t$.

We first give a proof for the special case of log loss for the reader that is only interested in that case. After that, we give a proof for unscaled $\alpha$-Tsallis loss which, when $\alpha = 1$, recovers our result for log loss up to a factor of 4.

\paragraph{Log loss.}
 
For log loss, observe that
\begin{align*}
\left| \loss(q, y) - \loss(r, y) \right|
\leq \left| \log \frac{q_y}{r_y} \right| 
\leq \log(1 + \varepsilon) 
\leq \varepsilon .
\end{align*}

\paragraph{Unscaled $\bm{\alpha}$-Tsallis loss.}

From \Cref{thm:Savage}, we can represent $\loss$ via its associated convex function $\psi$, giving
\begin{align*} 
\ell(q, y) - \ell(r, y) 
&= \psi(r) + \langle \nabla \psi(r), \mathbf{e}_y - r \rangle 
   - \psi(q) - \langle \nabla \psi(q), \mathbf{e}_y - q \rangle \\
&\leq \langle \nabla \psi(r) - \nabla \psi(q), \mathbf{e}_y - q \rangle \\
&\leq \|\nabla \psi(r) - \nabla \psi(q)\|_{\ast} \|\mathbf{e}_y - q\| ,
\end{align*}
where the first inequality is from 
$\psi(q) \geq \psi(r) + \langle \nabla \psi(r), q - r \rangle$. 
Note that the roles of $r$ and $q$ are symmetric, so the bound also holds for the absolute value. 
Now, letting $\|\cdot\| = \|\cdot\|_1$, we have $\|\mathbf{e}_y - q\| \leq 2$. Next, we show that
\begin{align*}
\| \nabla \psi(r) - \nabla \psi(q)\|_{\ast}
\leq \alpha (3 - \alpha) \varepsilon .
\end{align*}

Recall that 
$\psi(x) = \frac{1}{1 - \alpha} \bigl( 1 - \sum_{j=1}^d x_j^{\alpha} \bigr)$ 
and thus 
$\nabla \psi(x) 
 = \left( \frac{\alpha}{\alpha - 1} x_j^{\alpha - 1} \right)_{j \in [d]}$. 
Hence,
\begin{align*}
\|\nabla \psi(r) - \nabla \psi(q)\|_\infty
= \frac{\alpha}{\alpha-1} 
  \max_{j \in [d]} \left| r_j^{\alpha-1} - q_j^{\alpha-1} \right| .
\end{align*}

Consider an arbitrary coordinate $j \in [d]$. We consider the cases $q_j \leq r_j$ and $q_j \geq r_j$ in turn and show that both cases result in
\begin{align*}
\frac{|r_j - q_j|^{\alpha - 1}}{\alpha - 1} 
\leq \frac{1 - (1 - \varepsilon)^{\alpha - 1}}{\alpha - 1} .
\end{align*}
Indeed, if $q_j \leq r_j$, then letting $z = \frac{q_j}{r_j} \in [1 - \varepsilon, 1]$, it holds that
\begin{align*}
\frac{|r_j - q_j|^{\alpha - 1}}{\alpha - 1} 
= r_j^{\alpha - 1} \frac{1 - z^{\alpha - 1}}{\alpha - 1} 
\leq \frac{1 - (1 - \varepsilon)^{\alpha - 1}}{\alpha - 1} .
\end{align*}
Similarly, if $q_j \geq r_j$, then letting $z = \frac{r_j}{q_j} \in [1 - \varepsilon, 1]$, it holds that
\begin{align*}
\frac{|q_j - r_j|^{\alpha - 1}}{\alpha - 1} 
= q_j^{\alpha - 1} \frac{1 - z^{\alpha - 1}}{\alpha - 1} 
\leq \frac{1 - (1 - \varepsilon)^{\alpha - 1}}{\alpha - 1} .
\end{align*}

We define the function 
$f(\varepsilon) = (1 - \varepsilon)^{\alpha-1}$ 
and re-express it as 
$f(\varepsilon) = f(0) + f'(0) \varepsilon + f''(\xi) \frac{\varepsilon^2}{2}$ 
for some $\xi \in [0, \varepsilon]$. We have $f(0) = 1$. Also, using 
$f'(\varepsilon) = -(\alpha - 1) (1 - \varepsilon)^{\alpha-2}$ 
gives 
$f'(0) = -(\alpha - 1)$. 
Finally, $f''(\varepsilon) = (\alpha - 2) (\alpha - 1) (1 - \varepsilon)^3$. Therefore, for some $\xi \in [0, \varepsilon]$,
\begin{align*}
\frac{1 - (1 - \varepsilon)^{\alpha - 1}}{\alpha - 1} 
&= \frac{(\alpha - 1) \varepsilon - (\alpha - 2)(\alpha - 1) (1 - \xi)^{\alpha - 3} \cdot \frac{\varepsilon^2}{2}}
        {\alpha - 1} \\
&= \varepsilon + (2 - \alpha) (1 - \xi)^{\alpha - 3} \cdot \frac{\varepsilon^2}{2} \\
&\leq \varepsilon + (2 - \alpha) \frac{1}{(1 - \xi)^2} \frac{\varepsilon^2}{2} \\
&\leq \varepsilon + (2 - \alpha) \varepsilon ,
\end{align*}
where the last line uses $\xi \leq \varepsilon \leq \frac{1}{2}$. 
Overall, we have shown that
\begin{align*}
\|\nabla \psi(r) - \nabla \psi(q)\|_\infty 
\leq \alpha (3 - \alpha) \varepsilon .
\end{align*}
\end{proof}

Let us assume that we choose the binning from the proof of \cref{lemma:binning-new} so that the per-round approximation error is $O(\varepsilon)$. Then letting $N_\varepsilon$ be the number of bins, it is possible to show the following result. We omit factors polynomial in $d$ for simplicity.
\begin{theorem} \label{thm:calibeating}
Let $\alpha \in [1, 2]$ and $T \geq e N_\varepsilon$. Then under unscaled $\alpha$-Tsallis loss, taking the no-regret algorithm to be the FTRL-based U-calibration algorithm from \cref{sec:regret-bounds} (i.e., take $p_t$ as in \eqref{eqn:ftrl} with $\eta = 1$),
\begin{align*}
\sum_{t=1}^T \loss(q_t, y_t) - \sum_{t=1}^T \loss(\pcalibeat_t, y_t) 
\geq  \sum_{B \in \mathcal{B}} 
       \sum_{t \in \mathcal{T}_B} D_\psi(\bar{p}_T(B), \tilde{x}_B) 
       - O \left( 
               T \varepsilon + N_\varepsilon \log(T/N_\varepsilon) 
           \right) .
\end{align*}
Then, setting $\varepsilon = T^{-\frac{1}{d+1}} \log T$, for $T \geq e^{d+1}$,
\begin{align*}
\sum_{t=1}^T \loss(q_t, y_t) - \sum_{t=1}^T \loss(\pcalibeat_t, y_t) 
&\geq  \sum_{B \in \mathcal{B}} 
       \sum_{t \in \mathcal{T}_B} D_\psi(\bar{p}_T(B), \tilde{x}_B) - O(T^{d/(d+1)} \log T) .
\end{align*}
\end{theorem}

For completeness, we give a proof of the first inequality of the theorem in \cref{app:calibeating}. The second inequality just involves some algebra.

\section{Proof sketch of \cref{lemma:btrl-regret-equality}} \label{sec:proof-sketch-btrl-regret-equality}

In this section, we sketch our proof of \cref{lemma:btrl-regret-equality}. Along the way, we develop two results --- \cref{lemma:generalized-sn-thing} and \cref{lemma:weighted-sum-divergences} --- that strictly extend arguments of \cite{foster2023calibeating} and may be of interest in their own right. Because of how we prove \cref{lemma:btrl-regret-equality}, for this section only it will be convenient to adopt the notation $x_1, \ldots, x_n$ to refer to a sequence of $n$ outcomes (each represented as a ``one-hot'' $d$-dimensional vector). In addition, for any $k \in [n]$, we define the empirical mean $\bar{x}_k = \frac{1}{k} \sum_{i=1}^k x_i$ (which is an empirical frequency vector).

The spiritual predecessor of our analysis is Proposition 2 of \cite{foster2023calibeating}, which they used to calibeat in the case of squared loss. Their Proposition 2 gives the following online formula for the variance:
\begin{align*}
\sum_{i=1}^n \|x_i - \bar{x}_n\|_2^2 
= \sum_{i=1}^n \left( 1 - \frac{1}{i} \right) 
               \|x_i - \bar{x}_{i-1}\|_2^2 .
\end{align*}
As mentioned and shown by \cite{foster2023calibeating}, the above online formula for the variance easily follows from the variance update formula (see equation (10) of \cite{foster2023calibeating} or Formula I of \cite{welford1962note} for an earlier reference) which, defining $s_n = \sum_{i=1}^n \|x_i - \bar{x}_n\|^2$, is
\begin{align*}
s_n 
= s_{n-1} 
  + \left( 1 - \frac{1}{n} \right) \| x_n - \bar{x}_{n-1}\|^2 .
\end{align*}
This formula is easily established using the law of total variance (see \eqref{eqn:law-total-var} below).

Our main work in establishing \cref{lemma:btrl-regret-equality} is establishing a generalized version of the variance update formula and then using that result to obtain a generalized version of the online formula for the variance. The foundation for our generalized analysis is a powerful result\footnote{For more background, the interested reader is directed to the works of \cite{pfau2025generalized} and \cite{buja2005loss}.} of \cite{gupta2022ensembles}: for any Bregman divergence $D_\psi$, there is a generalized law of total variance. 
For random variables $X$ and $Y$, recall that the usual law of total variance states that
\begin{align}
\Var(X) = \E[\Var(X \mid Y)] + \Var(\E[X \mid Y]). \label{eqn:law-total-var}
\end{align}
The next lemma states the generalized law of total variance.
\begin{lemma}[Generalized Law of Total Variance \citep{gupta2022ensembles}] \label{lemma:gen-law-total-var}
Let $X$ and $Y$ be random variables. 
Relative to Bregman divergence $D_\psi$, define the generalized variance of $X$ as
\begin{align*}
\V[X] := \E[D_\psi(X,\E[X])]
\end{align*}
and the generalized conditional variance of $X$ given $Y$ as
\begin{align*}
\V[X \mid Y] := \E[D_\psi(X,\E[X \mid Y]) \mid Y] .
\end{align*}
Then
\begin{align*}
\V[X] = \E[\V[X \mid Y]] + \V[\E[X \mid Y]] .
\end{align*}
\end{lemma}

In preparation for our generalization of the variance update formula, we set up some notation. For any $n$, define for positive weights $w_1, \ldots, w_n$:
\begin{align*}
W_n := \sum_{j=1}^n w_j 
\qquad 
\bar{x}_n := \frac{\sum_{j=1}^n w_j x_{j}}{W_n} 
\qquad 
s_n := \sum_{i=1}^n w_i D_\psi(x_i, \bar{x}_n) .
\end{align*}

\begin{lemma} \label{lemma:generalized-sn-thing}
It holds that
\begin{align*}
s_n = s_{n-1} + w_n D_\psi(x_n, \bar{x}_n) + W_{n-1} D_\psi(\bar{x}_{n-1}, \bar{x}_n) .
\end{align*}
\end{lemma}
Note that \cref{lemma:generalized-sn-thing} generalizes the standard variance update formula both by allowing for general weights and, more significantly, allowing for generalized variances. Taking $w_i = 1$ and $\psi = \|\cdot\|_2$ recovers the standard formula. 
After completing our results, we noticed that \cite{foster2023calibeating}, in equation 36 of their Appendix A.6, had already generalized the standard variance update formula 
to allow for general weights (see also their Proposition 14 for their corresponding generalization of the online formula for the variance). Interestingly, Foster and Hart made this generalization to handle fractional binnings, whereas we made this generalization to handle loss-based regularization. Our results are further generalized by handling generalized variances (based on general Bregman divergences).

\begin{proof}(of \cref{lemma:generalized-sn-thing})
Let $X$ to be a random variable with law
\begin{align*}
\Pr(X = x_i) = \frac{w_i}{W_n} \qquad \text{for all } i \in [n] ,
\end{align*}
and let $Y = \ind{X = x_n}$.

Momentarily, we will establish the following three relationships:
\begin{align}
\V[X] 
&= \sum_{i=1}^n \frac{w_i}{W_n} D_\psi(x_i, \bar{x}_n) \label{eqn:first-rel} \\
\V \bigl[ \E [ X \mid Y ] \bigr] 
&= \frac{w_n}{W_n} D_\psi(x_n, \bar{x}_n) + \frac{W_{n-1}}{W_n} D_\psi(\bar{x}_{n-1}, \bar{x}_n) \label{eqn:second-rel} \\
\E \bigl[ \V[ X \mid Y ] \bigr] 
&= 
\sum_{i=1}^{n-1} \frac{w_i}{W_n} D_\psi(x_i, \bar{x}_{n-1}) \label{eqn:third-rel} .
\end{align}

Using these relationships together with the generalized law of total variance (\cref{lemma:gen-law-total-var}),
\begin{align*}
\V[X] = \E \bigl[ \V[X \mid Y ] \bigr] + \V \bigl[ \E [X \mid Y ] \bigr]
\end{align*}
gives (multiplying all terms by $W_n$)
\begin{align*}
\underbrace{\sum_{i=1}^n w_i D_\psi(x_i, \bar{x}_n)}_{s_n} 
= 
\underbrace{\sum_{i=1}^{n-1} w_i D_\psi(x_i, \bar{x}_{n-1})}_{s_{n-1}} 
+ w_n D_\psi(x_n, \bar{x}_n) + W_{n-1} D_\psi(\bar{x}_{n-1}, \bar{x}_n) ,
\end{align*}
which matches the lemma statement.

It remains to show \eqref{eqn:first-rel} through \eqref{eqn:third-rel}. 
Equation \eqref{eqn:first-rel} is from the definition of $X$ and $\V$:
\begin{align*}
\V[X] 
= \E \bigl[ D_\psi(X, \E[X]) \bigr] 
= \sum_{i=1}^n \frac{w_i}{W_n} D_\psi(x_i, \bar{x}_n) \bigr] .
\end{align*}

To see \eqref{eqn:second-rel}, we first compute $\E[X \mid Y]$:
\begin{align*}
\E [ X \mid Y ] 
= 
\begin{cases}
x_n & \text{if } Y = 1 ; \\
\bar{x}_{n-1} & \text{if } Y = 0 .
\end{cases}
\end{align*}
Next, from the definition of $\V$,
\begin{align*}
\V \bigl[ \E [X \mid Y] \bigr] 
&= \E \left[ D_\psi \left( \E [X \mid Y], \E \bigl[ \E [X \mid Y] \bigr] \right) \right] \\
&= \Pr(Y = 1) \cdot D_\psi(\E [X \mid Y = 1], \E[X]) 
   + \Pr(Y = 0) \cdot D_\psi(\E [X \mid Y = 0], \E[X]) \\
&= \frac{w_n}{W_n} \cdot D_\psi(x_n, \bar{x}_n) 
   + \frac{W_{n-1}}{W_n} \cdot D_\psi(\bar{x}_{n-1}, \bar{x}_n) .
\end{align*}

Finally, we show \eqref{eqn:third-rel}. First, observe that
\begin{align*}
\V[ X \mid Y ] 
&= \E \bigl[ D_\psi(X , \E[X \mid Y]) \mid Y \bigr] \\
&= 
\begin{cases}
0 & \text{if } Y = 1 ; \\
\displaystyle \sum_{i=1}^{n-1} \frac{w_i}{W_{n-1}} D_\psi(x_i, \bar{x}_{n-1}) & \text{if } Y = 0 .
\end{cases}
\end{align*}
Taking the expectation gives
\begin{align*}
\E \bigl[ \V[ X \mid Y ] \bigr] 
&= \frac{w_n}{W_n} \cdot 0 + \frac{W_{n-1}}{W_n} \sum_{i=1}^{n-1} \frac{w_i}{W_{n-1}} D_\psi(x_i, \bar{x}_{n-1}) \\
&=  \sum_{i=1}^{n-1} \frac{w_i}{W_n} \cdot D_\psi(x_i, \bar{x}_{n-1}) ,
\end{align*}
as desired.
\end{proof}

Next, we present a generalization of the online formula for the variance.

\begin{lemma} \label{lemma:weighted-sum-divergences}
For any positive weights $w_1, w_2, \ldots, w_n$, it holds that
\begin{align*}
\sum_{i=2}^n w_i D_\psi(x_i, \bar{x}_i) 
= \sum_{i=2}^n w_i D_\psi(x_i, \bar{x}_n)
   - \sum_{i=2}^n W_{i-1} D_\psi(\bar{x}_{i-1}, \bar{x}_i) .
\end{align*}
\end{lemma}

\begin{proof}(of \cref{lemma:weighted-sum-divergences})
Rearranging the result of Lemma~\ref{lemma:generalized-sn-thing} gives
\begin{align*}
w_n D_\psi(x_n, \bar{x}_n) = s_n - s_{n-1} - W_{n-1} D_\psi(\bar{x}_{n-1}, \bar{x}_n)
\end{align*}
Summing over $i$ from $2$ to $n$ gives
\begin{align*}
\sum_{i=2}^n w_i D_\psi(x_i, \bar{x}_i) 
&= \sum_{i=2}^n \left[ s_i - s_{i-1} - W_{i-1} D_\psi(\bar{x}_{i-1}, \bar{x}_i) \right] \\
&= s_n - s_1 - \sum_{i=2}^n W_{i-1} D_\psi(\bar{x}_{i-1}, \bar{x}_i) \\
&= \sum_{i=2}^n w_i D_\psi(x_i, \bar{x}_n) - \sum_{i=2}^n W_{i-1} D_\psi(\bar{x}_{i-1}, \bar{x}_i) ,
\end{align*}
which gives the corollary.
\end{proof}

With the above two lemmas in place, we are ready to prove \cref{lemma:btrl-regret-equality}. 

\begin{proof}(of \cref{lemma:btrl-regret-equality})
We start from \cref{lemma:weighted-sum-divergences} with $n = d + T$. Let the outcome sequence $(x_1, \ldots, x_{T+d})$ be equal to $(\mathbf{e}_1, \ldots, \mathbf{e}_d, y_1, \ldots, y_T)$. Finally, for the first $d$ rounds, set the weight as $w_j = \frac{1}{\eta}$, and for the remaining $T$ rounds, set the weight as $w_j = 1$. We first rewrite the above equality by splitting the various sums, using our settings for the weights, and replacing the outcomes by using elements of the type either $\mathbf{e}_j$ or $y_t$:
\begin{align*}
\frac{1}{\eta} \sum_{j=2}^d D_\psi \left( \mathbf{e}_j, \bar{x}_j \right)
+ \sum_{t=1}^T D_\psi(y_t, \bar{x}_{d+t}) 
&= \frac{1}{\eta} \sum_{j=2}^d D_\psi(\mathbf{e}_j, \bar{x}_{d+T}) 
      + \sum_{t=1}^T D_\psi(y_t, \bar{x}_{d+T}) \\
&\quad - \sum_{j=1}^{d-1} \frac{j}{\eta} D_\psi(\bar{x}_j, \bar{x}_{j+1}) 
             - \sum_{t=1}^T \left( \frac{d}{\eta} + t - 1 \right) D_\psi(\bar{x}_{d+t-1}, \bar{x}_{d+t}) .
\end{align*}
Next, observe that for $t \geq 0$, it holds that $\bar{x}_{d+t} = p_{t+1}$, which is the action played by BTRL in round $t$. Also, for $j \in [d]$, it holds that $\bar{x}_j = \frac{1}{j} \sum_{i=1}^j \mathbf{e}_i$, which, from our notation in the lemma statement, is just $p_{j/d}$. Using these replacements, the above can be re-expressed as
\begin{align*}
\frac{1}{\eta} \sum_{j=2}^d D_\psi \left( \mathbf{e}_j, p_{j/d} \right)
+ \sum_{t=1}^T D_\psi(y_t, p_{t+1}) 
&= \frac{1}{\eta} \sum_{j=2}^d D_\psi(\mathbf{e}_j, p_{T+1}) 
      + \sum_{t=1}^T D_\psi(y_t, p_{T+1}) \\
&\quad - \sum_{j=1}^{d-1} \frac{j}{\eta} D_\psi(p_{j/d}, p_{(j+1)/d}) 
             - \sum_{t=1}^T \left( \frac{d}{\eta} + t - 1 \right) D_\psi(p_t, p_{t+1}) ,
\end{align*}
which, using $D_\psi(y_t, p) = \loss(p, y_t)$ and $D_\psi(\mathbf{e}_j, p) = \loss(p, \mathbf{e}_j)$ for any $p$, is just
\begin{align*}
\frac{1}{\eta} \sum_{j=2}^d \loss(p_{j/d}, \mathbf{e}_j)
+ \sum_{t=1}^T \loss(p_{t+1}, y_t) 
&= \frac{1}{\eta} \sum_{j=2}^d \loss(p_{T+1}, \mathbf{e}_j) 
      + \sum_{t=1}^T \loss(p_{T+1}, y_t) \\
&\quad - \sum_{j=1}^{d-1} \frac{j}{\eta} D_\psi(p_{j/d}, p_{(j+1)/d}) 
             - \sum_{t=1}^T \left( \frac{d}{\eta} + t - 1 \right) D_\psi(p_t, p_{t+1}) .
\end{align*}
Rearranging gives
\begin{align*}
\sum_{t=1}^T \loss(p_{t+1}, y_t) 
- \sum_{t=1}^T \loss(p_{T+1}, y_t) 
&= \frac{1}{\eta} \sum_{j=2}^d \loss(p_{T+1}, \mathbf{e}_j)  
   - \frac{1}{\eta} \sum_{j=2}^d \loss(p_{j/d}, \mathbf{e}_j) \\
&\quad - \sum_{j=1}^{d-1} \frac{j}{\eta} D_\psi(p_{j/d}, p_{(j+1)/d}) 
       - \sum_{t=1}^T \left( \frac{d}{\eta} + t - 1 \right) D_\psi(p_t, p_{t+1}) ,
\end{align*}
as desired.
\end{proof}

\section{Discussion} \label{sec:discussion}

We have shown that for both the unscaled and scaled version of $\alpha$-Tsallis loss with $\alpha \in [1, 2]$, a single sequence of probability forecasts can obtain $O(d + d^{2-\alpha} \log T)$ regret. While it was known that there exists an algorithm that obtains $O(d \log T)$ regret for log loss, to our knowledge (from an extensive literature review) the best known regret bound (even without considering simultaneous regret) for the scaled $\alpha$-Tsallis loss with $\alpha \in (1, 2)$ was $O(d^2 \log T)$. This best known bound is from \cite{luo2024optimal}, and no prior results considered the unscaled version of $\alpha$-Tsallis loss (excluding the special cases of $\alpha = 1$ and $\alpha = 2$). For log loss, it is known that $O(d \log T)$ is the optimal rate. For squared loss, it is known that our additive $d$ term is should not be present. For $\alpha \in (1, 2)$, considering either the unscaled or scaled version of $\alpha$-Tsallis loss, it remains unclear what the optimal rate is. We conjecture that the optimal rate is $O(d^{2-\alpha} \log T)$ with no additive $d$ term. It would be interesting to prove a lower bound with this rate. Regarding our upper bound analysis, the additive $d$ term comes from our bound on the regret of BTRL relative to $p_{T+1}$ (\cref{lemma:penalty-term-tsallis-loss-improved}) and our bound on the smoothing term (\cref{lemma:smoothing-term-tsallis-loss}). Beyond the cases of Lipschitz losses and Tsallis losses, we hope that our new regret equality (\Cref{lemma:btrl-regret-equality}) may encounter further use. Unfortunately, in our analyses we were unable to fully leverage some negative terms appearing in this identity, but a more careful analysis could potentially sharpen existing regret bounds in our context, as well as others.

\subsubsection*{Acknowledgements}
M.~Fichtl's research was supported from ANID Anillo ACT210005 grant under a postdoctoral position at Universidad de Chile.
C.~Guzm\'an's research was
partially supported by ANID FONDECYT 1251029 grant, ANID Anillo ACT210005 grant, and National Center for Artificial Intelligence CENIA FB210017, Basal ANID. 
N.~Mehta was supported by the NSERC Discovery Grant RGPIN-2018-
03942. Much of this work happened while N.~Mehta was visiting C.~Guzm\'an at Pontificia Universidad Cat\'olica de Chile from February to April 2024.

\appendix

\section{Regret bound for Lipschitz losses when using our universal procedure}
\label{app:lipschitz-ftrl}

\begin{proposition} \label{prop:lipschitz-ftrl}
Assume that for all outcomes $y$, the loss $\loss(\cdot, y)$ is strictly proper and $G$-Lipschitz in a norm $\|\cdot\|$. Let $R$ be the $\|\cdot\|$-diameter of $\Delta_d$. Then taking $\eta = 1$, it holds for all $p \in \Delta_d$ that
\begin{align*}
\sum_{t=1}^T \bigl[ \loss(p_t, y_t) - \loss(p, y_t) \bigr] 
\leq G R (d + \log T ).
\end{align*}
\end{proposition}
\begin{proof}
Proposition~\ref{prop:decomposition} together with Lemma~\ref{lemma:btrl-regret-equality} (dropping the last two negative summations) gives
\begin{align}
\sum_{t=1}^T \bigl[ \loss(p_t, y_t) - \loss(p, y_t) \bigr] 
&\leq \sum_{t=1}^T \bigl[ \loss(p_t, y_t) - \loss_t(p_{t+1}, y_t) \bigr] 
+ \sum_{j=2}^d \loss(p_{T+1}, \mathbf{e}_j)   
- \sum_{j=2}^d \loss(p_{j/d}, \mathbf{e}_j) 
\nonumber \\
&\leq \sum_{t=1}^T G \|p_t - p_{t+1}\| + \sum_{j=2}^d G \|p_{T+1} - p_{j/d}\| \\
&\leq \sum_{t=1}^T G \|p_t - p_{t+1}\| + G R (d - 1) \label{eqn:lipschitz-ftrl} ,
\end{align}
where the second inequality uses the fact that the loss is Lipschitz.

Next, since the loss is proper and $\eta = \infty$, 
we have $p_t = \bar{p}_{t-1}$. Therefore,
\begin{align*}
\|p_t - p_{t+1}\| = \|\bar{p}_{t-1} - \bar{p}_t\| = \frac{1}{t} \|y_t - p_t\| \leq \frac{R}{t} .
\end{align*}
Hence, \eqref{eqn:lipschitz-ftrl} is at most
\begin{align*}
G R \sum_{t=1}^T \frac{1}{t} 
\leq G R (1 + \log T + d - 1) 
=  G R (d + \log T) .
\end{align*}
\end{proof}

\section[Proofs for \cref{sec:tsallis-regret-lemmas} (Regret for unscaled $\alpha$-Tsallis losses)]{Proofs for \cref{sec:tsallis-regret-lemmas} (Regret for scaled $\bm{\alpha}$-Tsallis losses)}
\label{app:tsallis-proofs}

\subsection{Stability term}

\begin{proof}(of \cref{lemma:stability-term-tsallis-loss})
We use the convexity of the loss to first bound:
\begin{align*}
\bigl( \loss(p_t, y_t) - \loss(p_{t+1}, y_t) \bigr)
\leq \langle \nabla \ell(p_t,y_t),p_t-p_{t+1}\rangle .
\end{align*}
    
Next we compute the vectors above,
\begin{align*}
\nabla \ell(p_t,y_t) 
&= \alpha 
   \bigl(
       p_j^{\alpha-1} - p_{y_t}^{\alpha-2} \delta_{j=y_t} 
   \bigr)_{j\in[d]} \\
p_t-p_{t+1} 
&= \frac{1}{t + d/\eta}
   \bigl( p_{j,t} - \delta_{j=y_t} \bigl)_{j \in [d]},
\end{align*}
as well as their inner product:
\begin{align*}
\langle \nabla \ell(p_t, y_t), p_t - p_{t+1} \rangle 
\leq \frac{\alpha}{t + d/\eta} 
     \Bigl(
         \sum_{j \in [d]} p_{j,t}^{\alpha} 
         - 2 p_{y_t,t}^{\alpha-1} 
         + p_{y_t,t}^{\alpha-2}
     \Bigr)
\leq \frac{\alpha}{t + d/\eta} \Big( 1 + p_{y_t,t}^{\alpha-2} \Big) .
\end{align*}
Using this last bound, we proceed to bound the stability term:
\begin{align*}
\sum_{t=1}^T \bigl( \loss(p_t, y_t) - \loss(p_{t+1}, y_t) \bigr)
&\leq \sum_{t=1}^T \frac{\alpha}{t + d/\eta} \Big( 1 + p_{y_t, t}^{\alpha-2} \Big) \\
&= \alpha \Bigl[ \log(T) 
   + \sum_{t=1}^T \frac{1}{t + d/\eta}
                  \Bigl(
                      \frac{t-1+d/\eta}{c_{y_{t-1},t-1} + 1/\eta}
                  \Bigr)^{2-\alpha}
          \Bigr].
\end{align*}
The last term can be bounded using H\"older's inequality as follows:
\begin{align*}
&\sum_{t=1}^T \frac{1}{(t + d/\eta)^{\alpha-1}}
              \frac{1}{(c_{y_{t-1},t-1} + 1/\eta)^{2-\alpha}} \\
&\leq \Bigl( \sum_{t=1}^T \frac{1}{t + d/\eta} \Bigr)^{\alpha-1} 
      \Bigl(
          \sum_{t=1}^T\frac{1}{c_{y_{t-1},t-1}+1/\eta}
      \Bigr)^{2-\alpha}\\
&\leq \log^{\alpha-1} \left( \frac{\eta T}{d} + 1 \right) 
       \cdot \left( 
                 d \left( \eta 
                          + \log \left( \frac{\eta T}{d} + 1 \right)
                   \right)
             \right)^{2-\alpha} \\
&\leq d^{2 - \alpha} 
      \left(
          \log \left( \frac{\eta T}{d} + 1 \right) 
          + \eta^{2 - \alpha}
            \log^{\alpha-1} \left( \frac{\eta T}{d} + 1 \right) 
      \right) .
\end{align*}  
\end{proof}

\subsection[Regret of BTRL relative to $p_{T+1}$]{Proof for regret of BTRL relative to $\bm{p_{T+1}}$}

\begin{proof}(of \cref{lemma:penalty-term-tsallis-loss-improved})
Consider outcome sequence $(\tilde{y}_t)_{t \in [d+T]} = (1, 2, \ldots, d, y_1, y_2, \ldots, y_T)$. 
Let $\tilde{p}_{t+1}$ be the action of Be The Leader (BTL) on this outcome sequence. 

From the BTL Lemma (e.g.~Lemma 3.1 of \cite{cesabianchi2006prediction}),
\begin{align*}
\sum_{t=1}^{d+T} \loss(\tilde{p}_{t+1}, \tilde{y}_t) 
- \sum_{t=1}^{d+T} \loss(p_{T+1}, \tilde{y}_t) 
\leq 0 ,
\end{align*}
which we rearrange as
\begin{align*}
\sum_{t=d+1}^{d+T} \loss(\tilde{p}_{t+1}, \tilde{y}_t) 
- \sum_{t=d+1}^{d+T} \loss(p_{T+1}, \tilde{y}_t) 
\leq 
\sum_{t=1}^d \loss(p_{T+1}, \tilde{y}_t) 
- \sum_{t=1}^d \loss(\tilde{p}_{t+1}, \tilde{y}_t) .
\end{align*}
Unpacking notation, this inequality is equivalent to\footnote{Note that we could conclude this same inequality by applying \cref{prop:doubly-generalized-from-gen-var}; indeed, that result gives a better inequality from the facts that $\loss(p_{1/d}, \mathbf{e}_1) = \loss(\mathbf{e}_1, \mathbf{e}_1) = 0$ and $\loss(p_{T+1}, \mathbf{e}_1) \geq 0$ plus the fact that the last two summation terms of \cref{prop:doubly-generalized-from-gen-var} are nonpositive. \label{fn:better}}
\begin{align*}
\sum_{t=1}^T \loss(p_{t+1}, y_t) - \sum_{t=1}^T \loss(p_{T+1}, y_t) 
\leq \sum_{j=1}^d \loss(p_{T+1}, \mathbf{e}_j) - \sum_{j=1}^d \loss(p_{j/d}, \mathbf{e}_j) .
\end{align*}

The right-hand side is equal to
\begin{align*}
&\sum_{j=1}^d \left( \frac{1}{1 - \alpha} \left( \alpha p_{j,T+1}^{\alpha-1} - 1 \right) + \sum_{i=1}^d p_{i,T+1}^\alpha \right)
- \sum_{j=1}^d \left( \frac{1}{1 - \alpha} \left( \alpha (p_{j/d})_j^{\alpha-1} - 1 \right) + \sum_{i=1}^d (p_{j/d})_i^\alpha \right) \\
&= 
\underbrace{\frac{\alpha}{1-\alpha} \sum_{j=1}^d \left( p_{j,T+1}^{\alpha-1} - (p_{j/d})_j^{\alpha-1} \right)}_{\textcolor{red}{\bigstar}} 
+ \underbrace{d \sum_{i=1}^d p_{i,T+1}^\alpha - \sum_{j=1}^d \sum_{i=1}^d (p_{j/d})_i^\alpha}_{\textcolor{blue}{\bigstar}} .
\end{align*}
For the term $\textcolor{blue}{\bigstar}$, we use the trivial upper bound of $d$.

We claim that the term $\textcolor{red}{\bigstar}$ is $O(d^{2-\alpha} \log T)$.
Let's try to show the claim. Observe that
\begin{align*}
&\frac{1}{1-\alpha} \sum_{j=1}^d \left( p_{j,T+1}^{\alpha-1} - (p_{j/d})_j^{\alpha-1} \right) \\
&= \frac{1}{\alpha-1} \sum_{j=1}^d \left( (p_{j/d})_j^{\alpha-1} - p_{j,T+1}^{\alpha-1} \right) \\
&= \frac{1}{\alpha-1} \left( \sum_{j=1}^d \left( \frac{1}{j} \right)^{\alpha-1} - \left( \frac{T+1/\eta}{T+d/\eta} \right)^{\alpha-1} - (d-1) \left( \frac{1/\eta}{T + d/\eta} \right)^{\alpha-1} \right) \\
&\leq \frac{1}{\alpha-1} \left( \sum_{j=1}^d \left( \frac{1}{j} \right)^{\alpha-1}  - d \left( \frac{1/\eta}{T + d/\eta} \right)^{\alpha-1} \right) \\
&= \frac{1}{\alpha-1} \sum_{j=1}^d \left( \left( \frac{1}{j} \right)^{\alpha-1}  - \left( \frac{1/\eta}{T + d/\eta} \right)^{\alpha-1} \right) .
\end{align*}
Now, we claim the above quantity is at most $d^{2-\alpha} \log T$. To this end, we divide $d^{2-\alpha}$ and try to show that the result is at most $\log T$. So, we wish to show that the following expression is at most $\log T$:
\begin{align}
\frac{1}{d} \cdot \frac{1}{\alpha-1} \sum_{j=1}^d \left( \left( \frac{d}{j} \right)^{\alpha-1}  - \left( \frac{d/\eta}{T + d/\eta} \right)^{\alpha-1} \right) . \label{eqn:1-over-d-convex}
\end{align}
Suppose that the above expression is convex in $\alpha$ on the range $\alpha \in [1, 2]$. Then it would suffice to upper bound the expression at its endpoints, $\alpha = 1$ and $\alpha = 2$. We claim that each of the individual terms in the summation is convex in $\alpha$. That is, we claim that for any $j \in [d]$, the expression
\begin{align}
\frac{1}{\alpha-1} \left( \left( \frac{d}{j} \right)^{\alpha-1}  - \left( \frac{d/\eta}{T + d/\eta} \right)^{\alpha-1} \right) . \label{eqn:is-it-convex?}
\end{align}
is convex in $\alpha$ for $\alpha \in [1, 2]$.

For now, let us suppose the claim is true. Then on the one hand, when $\alpha = 2$,
\begin{align*}
&\frac{1}{d} \cdot \frac{1}{\alpha-1} \sum_{j=1}^d \left( \left( \frac{d}{j} \right)^{\alpha-1}  - \left( \frac{d/\eta}{T + d/\eta} \right)^{\alpha-1} \right) \\
&= \frac{1}{d} \sum_{j=1}^d \left( \frac{d}{j} - \frac{d/\eta}{T + d/\eta} \right) \\
&\leq \sum_{j=1}^d \frac{1}{j} \\
&\leq 1 + \log d \\
&\leq 1 + \log T ,
\end{align*}
while on the other hand, when $\alpha = 1$, l'H\^opital's rule gives
\begin{align*}
&\frac{1}{d} \cdot \frac{1}{\alpha-1} \sum_{j=1}^d \left( \left( \frac{d}{j} \right)^{\alpha-1}  - \left( \frac{d/\eta}{T + d/\eta} \right)^{\alpha-1} \right) \\
&= \frac{1}{d} \cdot \lim_{\alpha \downarrow 1} \frac{d}{d \alpha} \sum_{j=1}^d \left( \left( \frac{d}{j} \right)^{\alpha-1}  - \left( \frac{d/\eta}{T + d/\eta} \right)^{\alpha-1} \right) \\
&= \frac{1}{d} \sum_{j=1}^d \left( \log \left( \frac{d}{j} \right) - \log \left( \frac{d/\eta}{T + d/\eta} \right) \right) \\
&\leq \log(d) + \log \left( \frac{\eta T + d}{d} \right) \\
&= \log(\eta T + d) \\
&\leq \log((\eta + 1) T) \\
&\leq \log(\eta + 1) + \log T .
\end{align*}

We now try to show that \eqref{eqn:is-it-convex?} is convex for $\alpha \in (1, 2]$. 
We rewrite \eqref{eqn:is-it-convex?} as
\begin{align*}
f(c) = \frac{a^c - b^c}{c}
\end{align*}
for some $a \geq 1$, $b \in (0, 1)$, and $c \in (0, 1]$. 
We will confirm that $f''(c) \geq 0$ for $c \in (0, 1]$. Computing the second derivative and multiplying by the nonnegative quantity $c^3$ gives
\begin{align*}
f''(c) \cdot c^3 
= c^2 (a^c \log^2(a) - b^c \log^2(b)) + 2 (a^c - b^c) - 2 c (a^c \log(a) - b^c \log(b)) =: g(c) .
\end{align*}
Next, $g'(c) = c^2 \left( a^c \log^3(a) - b^c \log^3(b) \right) \geq 0$ since $\log(a) \geq 0$ and $\log(b) \leq 0$. Hence, if $g(c) \geq 0$, we are done. Note that $g(0) = 0$, so we indeed have that $f$ is convex and hence \eqref{eqn:is-it-convex?} is convex.

Overall, we have shown that \eqref{eqn:1-over-d-convex} is at most $\max\{1 + \log T, \log(\eta + 1) + \log T\} = \log T + \max\{ 1, \log(\eta + 1) \}$, and so $\textcolor{red}{\bigstar}$ is at most $d^{2-\alpha} (\max\{ 1, \log(\eta + 1) \} + \log T)$. Together with $\textcolor{blue}{\bigstar} \leq d$, we have shown that BTRL's regret relative to $p_{T+1}$ is at most $d + d^{2-\alpha} (\max\{ 1, \log(\eta + 1) \} + \log T)$.
\end{proof}

\begin{proof}(of \cref{lemma:penalty-term-tsallis-loss-fallback})
From \Cref{lemma:btrl-regret-equality} with the negative terms dropped,
\begin{align*}
\sum_{t=1}^T \loss(p_{t+1}, y_t) - \sum_{t=1}^T \loss(p_{T+1}, y_t)
\leq \sum_{j=2}^d \loss(p_{T+1}, \mathbf{e}_j) 
\leq \sum_{j=1}^d \loss(p_{T+1}, \mathbf{e}_j) .
\end{align*}
By definition of the unscaled $\alpha$-Tsallis loss, the last line above is equal to
\begin{align*}
&\sum_{j=1}^d \frac{1}{1 - \alpha} 
             \left( \alpha p_{j,T+1}^{\alpha-1} - 1 \right) 
+ d \sum_{j=1}^d p_{j,T+1}^\alpha \\
&= \sum_{j=1}^d \frac{1}{\alpha - 1} 
             \left( 1 - \alpha p_{j,T+1}^{\alpha-1} \right) 
+ d \sum_{j=1}^d p_{j,T+1}^\alpha \\
&\leq 
\frac{d}{\alpha - 1} 
- \sum_{j=1}^d \frac{1}{\alpha - 1} \alpha p_{j,T+1}^{\alpha- 1} 
+ d \sum_{j=1}^d p_{j,T+1} \\
&= 
d \cdot \frac{\alpha}{\alpha - 1} 
- \sum_{j=1}^d \frac{1}{\alpha - 1} \alpha p_{j,T+1}^{\alpha- 1} \\
&\leq 
d \cdot \frac{\alpha}{\alpha - 1} .
\end{align*}
\end{proof}

\subsection{Smoothing term}

\begin{proof}(of \cref{lemma:smoothing-term-tsallis-loss})
To control the smoothing term, observe that
\begin{align*}
\sum_{t=1}^T \left[ \loss(x_t, p^{(\eta)}_{T+1}) - \loss(x_t, p^{(\infty)}_{T+1}) \right] 
= T \sum_{j=1}^d p_{j,T+1}^{(\infty)} \left[ \loss(j, p^{(\eta)}_{T+1}) - \loss(j, p^{(\infty)}_{T+1}) \right] 
\end{align*}

We begin by rewriting the above with the unscaled $\alpha$-Tsallis loss expanded and adopting the notation $p = p_{T+1}^{(\infty)}$ and $\tilde{p} = p_{T+1}^{(1)}$ (recall that we set $\eta = 1$):
\begin{align}
&T \sum_{j=1}^d p_j \left[ 
  \frac{1}{1 - \alpha} \left( \alpha \tilde{p}_j^{\alpha - 1} - 1 \right) + \sum_{k=1}^d \tilde{p}_k^\alpha
  -
  \frac{1}{1 - \alpha} \left( \alpha p_j^{\alpha - 1} - 1 \right) + \sum_{k=1}^d p_k^\alpha
  \right] \nonumber \\
  &=
T \sum_{j=1}^d p_j \left[ 
  \frac{1}{1 - \alpha} \left( \alpha \tilde{p}_j^{\alpha - 1} - 1 \right) 
  -
  \frac{1}{1 - \alpha} \left( \alpha p_j^{\alpha - 1} - 1 \right) 
  \right] \label{eqn:smoothing-I} \\
&\quad  + T \left[
      \sum_{k=1}^d \tilde{p}_k^\alpha
      - \sum_{k=1}^d p_k^\alpha
    \right] \label{eqn:smoothing-II}
\end{align}

We first upper bound \eqref{eqn:smoothing-II} since it is easier to control and then upper bound \eqref{eqn:smoothing-I}.

To control \eqref{eqn:smoothing-II}, we first define the function
\begin{align*}
f(x) = \left( \frac{T}{T + d} \cdot x  + \frac{1}{T + d} \right)^\alpha
\end{align*}
and observe that $\tilde{p} = f(p)$. Upper bounding \eqref{eqn:smoothing-II} is therefore equivalent to upper bounding
\begin{align*}
T \cdot \sum_{j=1}^d \left( f(p_j) - p_j^\alpha \right) .
\end{align*}
We will show that $f$ is $\alpha$-Lipschitz for $x \in [0, 1]$, after which light work yields a suitable bound. Taking the derivative, we have
\begin{align*}
f'(x) = \alpha \cdot \left( \frac{T}{T + d} \cdot x  + \frac{1}{T + d} \right)^{\alpha-1}  \cdot \frac{T}{T + d} ,
\end{align*}
which satisfies (since $\alpha \geq 1$)
\begin{align*}
|f'(x)| 
&\leq \alpha \cdot \left( \frac{T}{T + d} \cdot x  + \frac{1}{T + d} \right)^{\alpha-1} \\
&\leq \alpha \cdot \left( \frac{T + 1}{T + d} \right)^{\alpha-1} \\
&\leq \alpha .
\end{align*}
Hence, $f$ is indeed $\alpha$-Lipschitz. 
Now, observe that for any $p_j \in [0, 1]$, defining $\bar{p}_j := \left(p_j - \frac{1}{T+d} \right) \cdot \frac{T+d}{T}$ gives $f(\bar{p}_j) = p_j^\alpha$. Hence,
\begin{align*}
T \cdot \sum_{j=1}^d \left( f(p_j) - p_j^\alpha \right) 
&= T \cdot \sum_{j=1}^d \left( f(p_j) - f(\bar{p}_j) \right) \\
&\leq T \cdot \sum_{j=1}^d \alpha \left| p_j - \bar{p}_j \right| .
\end{align*}
Re-expressing $\bar{p}_j$ as $\bar{p}_j = p_j \cdot \left( 1 + \frac{d}{T} \right) - \frac{1}{T}$, we see that the above is at most
\begin{align}
T \cdot \sum_{j=1}^d \alpha \left| \frac{d}{T} p_j - \frac{1}{T} \right| 
&\leq \alpha T \cdot \left( \sum_{j=1}^d p_j \cdot \frac{d}{T} + \sum_{j=1}^d \frac{1}{T} \right) \nonumber \\
&= 2 \alpha d , \label{eqn:result-for-smoothing-II}
\end{align}
which is nice since it does not depend on $T$ and is linear in $d$.

Next, we control \eqref{eqn:smoothing-I}. 
We will give an upper bound on the function $h \colon [0, 1] \rightarrow \reals$, defined as
\begin{align*}
h_\alpha(x) = \frac{1}{\alpha - 1} \left( x^{\alpha - 1} - \left( \frac{T}{T + d} \cdot x + \frac{1}{T + d} \right)^{\alpha - 1} \right) .
\end{align*}
Observe that $h_\alpha(p_j) = \frac{1}{\alpha - 1} \left( p_j^{\alpha - 1} - \tilde{p}_{j-1}^{\alpha - 1} \right)$. 
It can easily be verified that $x < \frac{1}{d}$ implies that $h_\alpha(x) < 0$. Therefore, it suffices to consider the case of $x \geq \frac{1}{d}$. We will proceed by establishing that for all $\alpha \in [1, 2]$, the function $h_\alpha$ is maximized at $x = 1$. We then show that for $x = 1$, the function $\alpha \mapsto h_\alpha(1)$ is maximized at $\alpha = 1$, after which we compute the maximum.

Consider an arbitrary $\alpha \in [1, 2]$. We will show that the derivative $h'_\alpha(x)$ is increasing in $x$ on the interval $[1/d, 1]$, which would then imply that $h_\alpha$ is maximized at $x = 1$. So, let us compute the derivative:
\begin{align*}
h'_\alpha(x) = x^{\alpha - 2} - \left( \frac{T}{T + d} \cdot x + \frac{1}{T + d} \right)^{\alpha - 2} \cdot \frac{T}{T + d} .
\end{align*}

Now, observe that each of $f_\alpha(x) = x^{\alpha - 2}$ and $g_\alpha(x) = \left( \frac{T}{T + d} \cdot x + \frac{1}{T + d} \right)^{\alpha - 2} \cdot \frac{T}{T + d}$ is nonnegative. Therefore, if their ratio satisfies $\frac{f_\alpha(x)}{g_\alpha(x)} \geq 1$, it follows that $f_\alpha(x) \geq g_\alpha(x)$ and hence $h'_\alpha(x) = f_\alpha(x) - g_\alpha(x) \geq 0$, as desired. Showing that $\frac{f_\alpha(x)}{g_\alpha(x)} \geq 1$ is equivalent to showing that
\begin{align*}
\frac{T + d}{T} \geq \frac{x^{2 - \alpha}}{\left( \frac{T}{T + d} \cdot x + \frac{1}{T + d} \right)^{2 - \alpha}} ,
\end{align*}
or equivalently,
\begin{align*}
\left( \frac{T + d}{T} \right)^{1/(2 - \alpha)} \geq \frac{x}{\frac{T}{T + d} \cdot x + \frac{1}{T + d}} ,
\end{align*}
which is equivalent to
\begin{align*}
\left( \frac{T + d}{T} \right)^{1/(2 - \alpha)} \geq \frac{T + d}{T + \frac{1}{x}} .
\end{align*}
Now, among $x \in [1/d, 1]$, it is clear that the right-hand side is maximized for $x = 1$ and the left-hand side is minimized for $\alpha = 1$, which gives
\begin{align*}
\frac{T + d}{T} \geq \frac{T + d}{T + 1} ,
\end{align*}
which is obviously true. Therefore, we indeed have that for any $\alpha \in [1, 2]$, the derivative $h'_\alpha(x) \geq 0$ for $x \in [1/d, 1]$ and hence $h_\alpha$ is maximized when $x = 1$. It remains to show that $\bar{h}(\alpha) := h_\alpha(1)$ is maximized for $\alpha = 1$. This is easily verified as follows. First, note that
\begin{align*}
\bar{h}(\alpha) = \frac{1}{\alpha - 1} \left( 1 - \left( \frac{T+1}{T+d} \right)^{\alpha - 1} \right) 
\end{align*}
Next, we compute the derivative and show that it is nonpositive on the interval $\alpha \in (1, 2]$. The derivative is
\begin{align*}
\bar{h}'_\alpha(1) = -\frac{1}{(\alpha - 1)^2} \left( 1 - \left( \frac{T+1}{T+d} \right)^{\alpha - 1} \right) + \frac{1}{\alpha - 1} (-1) \log \left( \frac{T+1}{T+d} \right) \left( \frac{T+1}{T+d} \right)^{\alpha - 1} .
\end{align*}
Multiplying both sides by $(\alpha - 1)^2$ and rearranging gives
\begin{align*}
-\left( 1 - \left( \frac{T+1}{T+d} \right)^{\alpha - 1} \right) + (\alpha - 1) \log \left( \frac{T+d}{T+1} \right) \left( \frac{T+1}{T+d} \right)^{\alpha - 1} ,
\end{align*}
which, after multiplying by $\left( \frac{T+d}{T+1} \right)^{\alpha - 1}$ gives
\begin{align*}
&-\left( \frac{T+d}{T+1} \right)^{\alpha - 1} + 1 + (\alpha - 1) \log \left( \frac{T+d}{T+1} \right) \\
&= -\left( \frac{T+d}{T+1} \right)^{\alpha - 1} + 1 +  \log \left( \left( \frac{T+d}{T+1} \right)^{\alpha - 1} \right) .
\end{align*}
Finally, the inequality $1 + x \leq e^x$ implies that 
\begin{align*}
\log \left( \left( \frac{T+d}{T+1} \right)^{\alpha - 1} \right)
\leq \left( \frac{T+d}{T+1} \right)^{\alpha - 1} - 1 ,
\end{align*}
and so the derivative of $\bar{h}(\alpha)$ is indeed nonpositive on the interval $\alpha \in (1, 2]$.

Finally, from l'H\^opital's rule,
\begin{align*}
\bar{h}_1(1) 
&= \frac{1 - \left( \frac{T+1}{T+d} \right)^{\alpha - 1}}{\alpha - 1} \\
&= -\log \left( \frac{T+1}{T+d} \right) \\
&= \log \left( \frac{T+d}{T+1} \right) \\
&= \log \left( 1 + \frac{d-1}{T+1} \right) \\
&\leq \frac{d-1}{T+1} .
\end{align*}
Plugging this result into \eqref{eqn:smoothing-I} gives
\begin{align*}
&T \sum_{j=1}^d p_j \left[ 
  \frac{1}{1 - \alpha} \left( \alpha \tilde{p}_j^{\alpha - 1} - 1 \right) 
  -
  \frac{1}{1 - \alpha} \left( \alpha p_j^{\alpha - 1} - 1 \right) 
  \right] \\
&\leq \alpha T \sum_{j=1}^d p_j \cdot \frac{d-1}{T+1} \\
&= \alpha \frac{T}{T+1} \cdot (d - 1) \\
&\leq \alpha \cdot (d - 1) .
\end{align*}

Using this bound together with \eqref{eqn:result-for-smoothing-II} gives us the bound of $\alpha (d - 1) + 2 \alpha d \leq 3 \alpha d$ for the smoothing term.
\end{proof}

\section{Proofs for Calibeating (\cref{sec:calibeating})}
\label{app:calibeating}

\begin{proof}(of \cref{thm:calibeating})
As mentioned in the main text, we will only prove the first inequality.

First, from \cref{thm:tsallis-loss}, for each bin we obtain an additive term of order $\log |\mathcal{T}_B|$. Next, taking an indexing of the bins, for bin empirical frequencies $r_1, \ldots, r_{N_\varepsilon}$, maximizing the quantity $\sum_{i=1}^{N_\varepsilon} \ind{r_i > 0} \log (T r_i)$. Now, suppose that $k \leq N_\varepsilon$ bins have positive frequencies; without loss of generality, let these bins be the first $k$ bins. For fixed $k$, the maximization is equivalent to maximizing $\sum_{i=1}^k \log (T r_i)$, itself equivalent to maximizing $\sum_{i=1}^k \log r_i$. Since the objective is a symmetric, concave function, the maximum is achieved by the uniform distribution, and therefore the maximum value of $\sum_{i=1}^k \log (T r_i)$ is $k \log \frac{T}{k}$. The latter function is non-decreasing for $T \geq e k$. So, if we assume that $N_\varepsilon \leq \frac{T}{e}$, then $k$ is maximized at $N_\varepsilon$.
\end{proof}

\vskip 0.2in
\bibliography{refs}

\end{document}